\documentclass[twoside,11pt]{article}
\usepackage{theapa}
\usepackage{geometry}
 
\usepackage{amssymb}
\usepackage{amsthm}
\usepackage{colonequals}
\usepackage{tikz}
\usepackage{aurical}
\usepackage[T1]{fontenc}

\newcommand{\parents}[1]{\mathrm{pa}(#1)}
\newcommand{\pro}{\mathbb{P}}
\newcommand{\pr}[1]{\mathbb{P}\!\left( #1 \right)}
\newcommand{\lpro}{\underline{\mathbb{P}}}
\newcommand{\lpr}[1]{\underline{\mathbb{P}}\!\left( #1 \right)}
\newcommand{\upr}[1]{\overline{\mathbb{P}}\!\left( #1 \right)}
\newcommand{\ex}[2]{\mathbb{E}_{#1}\!\left[ #2 \right]}
\newcommand{\lex}[1]{\underline{\mathbb{E}}\!\left[ #1 \right]}
\newcommand{\credalo}{\mathbb{K}}
\newcommand{\bigperiod}{\mbox{\large\bfseries .}}

\newtheorem{Definition}{Definition}
\newtheorem{Theorem}[Definition]{Theorem}
\newtheorem{Proposition}[Definition]{Proposition}

\newtheorem{Corollary}[Definition]{Corollary}
\theoremstyle{definition}
\newtheorem{Example}[Definition]{Example}
 
\begin{document}

\title{On the Semantics and Complexity of Probabilistic Logic Programs}

\author{%
%\name 
Fabio Gagliardi Cozman  \\  
%\addr 
Escola Polit\'ecnica, Universidade de S\~ao Paulo, Brazil  
\and
%\name 
Denis Deratani Mau\'a \\  
%\addr 
Instituto de Matem\'atica e Estat{\'\i}stica, Universidade de S\~ao Paulo, Brazil}
 
\maketitle

\begin{abstract}
We examine the meaning and the complexity of probabilistic logic programs that consist
of a set of rules and a set of independent probabilistic facts (that is, programs based
on Sato's distribution semantics). We focus on two semantics, respectively based on stable
and on well-founded models. We show that the semantics based on stable models (referred
to as the ``credal semantics'') produces sets of probability models that dominate infinitely monotone
Choquet capacities; we describe several useful consequences of this result. We then examine
the complexity of inference with probabilistic logic programs.
We distinguish between the complexity of inference when 
a probabilistic program and a query are given (the {\em inferential} complexity), and the 
complexity of inference when the probabilistic program is fixed and the query is given 
(the {\em query} complexity, akin to {\em data} complexity as used in database theory). 
We obtain results on the inferential and query complexity for acyclic, stratified, 
and cyclic propositional and relational programs; complexity reaches various levels 
of the counting hierarchy and even exponential levels.
\end{abstract}

%\begin{keyword}
%Probabilistic logic programs
%\sep Distribution semantics 
%\sep Complexity theory.
%\end{keyword}
%\end{frontmatter}
%% \linenumbers

\section{Introduction}\label{section:Introduction}

The combination of deterministic and uncertain reasoning has led to
many mixtures of logic and probability \cite{Halpern2003,Hansen96,Nilsson86}.
%Fragments of first-order logic have been mixed with independence
%assumptions so as to generate first-order Bayesian networks of various
%kinds \cite{Jaeger97,Koller98}.
%There are also several modeling languages that combine probabilities
%with logic programming, thus displaying features that complement
%fragments of first-order logic, such as transitive closure.
In particular, combinations of logic programming constructs and probabilistic assessments
have been pursued in several guises \cite{Fuhr95,Lukasiewicz98,Ng92,Poole93AI,Sato95}, 
and the topic has generated significant literature \cite{Raedt2010,Raedt2008Book}.

Among probabilistic logic programming languages, the approach started by
Poole's probabilistic Horn abduction \cite{Poole93AI} and Sato's distribution 
semantics~\cite{Sato95} has been very popular. Basically, there
a logic program is enlarged with independent probabilistic facts. 
For instance, consider a rule
\[
\mathsf{up} \colonminus \mathsf{actionUp}, \mathbf{not}\;\mathsf{disturbanceUp}\bigperiod  
\]
and probabilistic fact
\[
\pr{\mathsf{disturbanceUp}=\mathsf{true}}=0.1.
\]
Depending on $\mathsf{disturbanceUp}$,  $\mathsf{actionUp}$ may succeed or not in 
leading to $\mathsf{up}$.  

Sato's distribution semantics at first focused on {\em definite} programs, and was 
announced ``roughly, as distributions over least models''~\cite{Sato95}. 
Poole and Sato  originally emphasized   {\em acyclic} logic programs
\cite{Poole93AI,Poole2008,Sato95,Sato2001}, even though Sato did handle 
cyclic ones. Since then, there has been significant work on non-definite and 
on cyclic probabilistic logic programs under variants of the distribution semantics 
\cite{Hadjichristodoulou2012,Lukasiewicz2005ecsqaru,Riguzzi2015,Sato2005}.

In this paper we  examine the meaning and the computational 
complexity of probabilistic logic programs that extend Sato's distribution semantics. 
We look at standard function-free normal programs containing negation as failure 
and probabilistic facts. The goal is to compute an {\em inference}; that is, to compute 
the probability $\pr{\mathbf{Q}|\mathbf{E}}$, where both $\mathbf{Q}$ and 
$\mathbf{E}$ are sets of facts. The pair $(\mathbf{Q},\mathbf{E})$ is referred to as 
the {\em query}. We distinguish between the complexity of inference when a 
probabilistic program and a query are given (the {\em inferential} complexity),
and the complexity of inference when the probabilistic program is fixed
and the query is given (the {\em query} complexity).
Query complexity is similar to {\em data}
complexity as used in database theory, as we discuss later.  

We first examine acyclic programs; for those programs all existing semantics
coincide.  Given the well-known relationship
between acyclic probabilistic logic programs and Bayesian networks, it is not surprising
that inference for propositional acyclic programs is
$\mathsf{PP}$-complete. However, it {\em is} surprising that, as we show, inference with
{\em bounded arity} acyclic programs {\em without negation} is
$\mathsf{PP}^\mathsf{NP}$-equivalent, thus going up the counting hierarchy. 
And we show that acyclic programs without a bound on predicate arity take us to
$\mathsf{PEXP}$-completeness.

Many useful logic programs are cyclic; indeed, the use of recursion is at the heart of logic
programs and its various semantics \cite{Gelfond88,Gelder91}. Many applications,
such as non-recursive structural equation models~\cite{Berry84,Pearl2009} %\cite[Section 3.6]{Pearl2009}  
and  models with ``feedback''~\cite{Nodelman2002,Poole2013},
defy the acyclic character of Bayesian networks. 

We study cyclic normal logic programs in a few steps. First we look at the inferential 
and query complexity of  locally stratified programs. 
For these programs, again we see that most existing semantics coincide;
in particular semantics based on stable and well-founded models are identical.
To summarize, we show that the complexity of stratified programs is the same as 
the complexity of acyclic programs. 

We then move to general, possibly cyclic, programs. There are various semantics
for such programs, and relatively little discussion about them in the literature.
For instance, take a program consisting of two rules,
\begin{equation}
\label{equation:BasicCyclicProgram}
\mathsf{sleep} \colonminus \mathbf{not}\;\mathsf{work}, \mathbf{not}\;\mathsf{insomnia}\bigperiod 
\qquad
\mathsf{work} \colonminus \mathbf{not}\;\mathsf{sleep}\bigperiod
\end{equation}
and a fact associated with a probabilistic assessment: 
\[
\pr{\mathsf{insomnia}=\mathsf{true}}=0.3.
\]
With probability $0.3$, we have that $\mathsf{insomnia}$ is $\mathsf{true}$, and then 
$\mathsf{sleep}$ is $\mathsf{false}$ and $\mathsf{work}$ is $\mathsf{true}$. This is
simple enough. But with probability $0.7$, we have that $\mathsf{insomnia}$ is 
$\mathsf{false}$, and then the remaining two rules create a cycle: $\mathsf{sleep}$
depends on $\mathsf{work}$ and vice-versa. The question is how to define a semantics
when a cycle appears.

We focus on two semantics  for such programs, even though we mention a few others. 
First, we look at a semantics for probabilistic logic programs 
that can be extracted from the work of Lukasiewicz on probabilistic description
logics~\cite{Lukasiewicz2005ecsqaru,Lukasiewicz2007}.
His proposal is that a probabilistic logic program defines a set
of probability measures, induced by the various stable models of the underlying
normal logic program.
The second semantics we examine is based on the well-founded semantics
of normal logic programs: in this case there is always a single distribution induced by
a probabilistic logic program \cite{Hadjichristodoulou2012}. 

We first study Lukasiewicz's semantics, referred to as the ``credal semantics''.
We show that credal semantics produces 
sets of probability models that dominate infinitely monotone Choquet capacities; the latter
objects are relatively simple extensions of probability distributions and have been
often used in the literature, from random set theory to Dempster-Shafer theory.
We then derive results concerning inferential and query complexity. 
We show that the complexity of general probabilistic logic programs goes up the 
counting hierarchy, up to $\mathsf{PP}^{\mathsf{NP}^\mathsf{NP}}$ levels; 
overall the complexity of the well-founded semantics is in lower classes than
the complexity of the stable model semantics. 

The paper begins in Section \ref{section:Basics} with a review of logic
programming and complexity theory. 
Section \ref{section:PLPs} presents basic notions concerning
probabilistic logic programs and their semantics.
In Section \ref{section:Semantics} we contribute with a comparison between the
{\em credal} and the {\em well-founded} semantics.
Our main results appear in Sections \ref{section:Structure},
\ref{section:ComplexityAcyclicStratified},
\ref{section:ComplexityCredal} and \ref{section:ComplexityWellFounded}.
In Section \ref{section:Structure} we show that the credal semantics
of a probabilistic logic program is a set of probability measures induced
by a 2-monotone Choquet capacities. 
Sections \ref{section:ComplexityAcyclicStratified},
 \ref{section:ComplexityCredal} and \ref{section:ComplexityWellFounded}
 analyze the complexity of inferences under the credal and
the well-founded semantics.  The paper concludes, in Section
\ref{section:Conclusion}, with a summary of our contributions and a
discussion of future work. 

\section{Background}\label{section:Basics}

We briefly collect here some well known terminology and notation regarding
logic programming and complexity theory. Before we plunge into those topics,
we briefly fix notation on Bayesian networks as we will need them later.
A {\em Bayesian network} is a pair consisting of
 a directed acyclic graph $\mathbb{G}$ whose nodes are random variables,  
 and a joint probability distribution $\pro$ over all variables in the graph, such that 
 $\mathbb{G}$ and $\pro$ satisfy the ``Markov condition'' (that is, a random
variable is independent of its parents given its nondescendants)~\cite{Koller2009,Neapolitan2003,Pearl88Book}. 
%For a collection of measurable sets $A_1,\dots,A_n$, we then have
%\[
%\pr{X_1 \in A_1, \dots,X_n \in A_n} = 
%	\prod_{i=1}^n \pr{X_i \in A_i|\parents{X_i} \in {\textstyle \bigcap_{j:X_j \in \parents{X_i}} A_j}}
%\]
%whenever the conditional probabilities exist.
If all random variables are discrete, then one can specify ``local''  conditional probabilities
$\pr{X_i=x_i|\parents{X_i}=\pi_i}$, and the joint probability distribution is  necessarily
the product of these local probabilities:
\begin{equation}
\label{equation:BayesianNetwork}
\pr{X_1=x_1, \dots,X_n=x_n} = \prod_{i=1}^n \pr{X_i=x_i|\parents{X_i}=\pi_i},
\end{equation}
where $\pi_i$ is the projection of $\{x_1,\dots,x_n\}$ on the set of random variables $\parents{X_i}$;
whenever $X_i$ has no parents, $\pr{X_i=x_i|\parents{X_i}=\pi_i}$ stands for $\pr{X_i=x_i}$.

\subsection{Normal logic programs: syntax and semantics}

Take a vocabulary consisting of set of logical variable symbols
$X,  Y, \dots$, a set of predicate symbols $\mathsf{r},\mathsf{s}, \dots$, and
a set of constants $a, b,\dots$. A {\em term} is a constant or a
logical variable; an {\em atom} is written as
$\mathsf{r}(t_1,\dots,t_n)$, where $\mathsf{r}$ is a predicate of arity
$n$ and each $t_i$ is a term. A zero-arity atom is written simply as
$\mathsf{r}$.  An atom is {\em ground} if it does not contain logical
variables.  

A {\em normal logic program} consists of rules written as \cite{Dantsin2001}
\[
A_0   \colonminus    A_1, \dots, A_m, \mathbf{not} A_{m+1}, \dots, \mathbf{not} A_{n}\bigperiod
\]
where the $A_i$ are atoms and $\mathbf{not}$ is 
interpreted according to some selected semantics, as discussed
later.  The {\em head} of this rule is $A_0$; the remainder of the rule is its
{\em body}.  A rule without a body, written simply as $A_0\bigperiod$, is a
{\em fact}.  A {\em subgoal} in the body is either an atom $A$ (a {\em positive} subgoal) 
or $\mathbf{not} \; A$ 
(a {\em negative} subgoal). A program without negation is {\em
definite}, and a program without variables is {\em propositional}.

\begin{Example}\label{example:logicSmokers}
Here is  a program describing the relation between
smoking, stress, and social influence \cite{Fierens2015}:
  \[
    \begin{array}{l}
      \mathsf{smokes}(X) \colonminus \mathsf{stress}(X)\bigperiod \\
      \mathsf{smokes}(X) \colonminus \mathsf{influences}(Y,X), \mathsf{smokes}(Y)\bigperiod \\
      \mathsf{influences}(a,b)\bigperiod \;\;  \mathsf{influences}(b,a)\bigperiod \;\;    \mathsf{stress}(b)\bigperiod 
    \end{array}
  \]
This program is definite, but not propositional. 
$\hfill \Box$
\end{Example}

The {\em Herbrand base} of a program is the set of all ground atoms built
from constants and predicates in the program. We do not consider functions
in this paper, to stay with finite Herbrand bases.

A {\em substitution} is a (partial) function that maps logical variables into
terms. A {\em grounding} is a substitution mapping into constants. The
grounding of a rule is a ground rule obtained by applying the same
grounding to each atom. The grounding of a program is the propositional
program obtained by applying every possible grounding all rules, using
only the constants in the program (i.e., using only ground atoms in the
Herbrand base).
An atom $A$ {\em unifies} with an atom $B$ if there is a substitution
that makes both (syntactically) equal.

A {\em literal} $L$ is either an atom $A$ or a negated atom $\neg A$. 
A set of literals is {\em inconsistent} if $A$ and $\neg A$ belong to it. 
Given a normal logic program $\mathbf{P}$, 
a {\em partial interpretation} is a consistent set of literals whose atoms belong to the Herbrand base
of $\mathbf{P}$. An {\em interpretation} is a consistent set of literals such that every atom in the 
Herbrand base appears in a literal. An atom is $\mathsf{true}$ (resp., $\mathsf{false}$) in a 
(partial) interpretation if it appears in a non-negated (resp., negated) literal. 
A subgoal is $\mathsf{true}$ in an interpretation if it is an atom $A$ and $A$ belongs to the interpretation,
or the subgoal is $\mathbf{not} \; A$ and $\neg A$ belongs to the interpretation.
A grounded rule  is {\em satisfied} in a partial  interpretation if
its head is $\mathsf{true}$ in the interpretation, or any of its subgoals is $\mathsf{false}$
in the interpretation.
A {\em model} of $\mathbf{P}$ is an interpretation such that 
every grounding of a rule in $\mathbf{P}$ is satisfied.
A {\em minimal model} of $\mathbf{P}$ is a model with minimum number of non-negated literals.

\begin{figure}[t]
\centering
\begin{tikzpicture}[scale=1.1]
\node[draw,rectangle,rounded corners] (iba) at (0,1) {$\mathsf{influences}(b,a)$};
\node[draw,rectangle,rounded corners] (iab) at (3,1) {$\mathsf{influences}(a,b)$};
\node[draw,rectangle,rounded corners] (ta) at (-2.5,1) {$\mathsf{stress}(a)$};
\node[draw,rectangle,rounded corners] (tb) at (5.5,1) {$\mathsf{stress}(b)$};
\node[draw,rectangle,rounded corners] (sa) at (0,0) {$\mathsf{smokes}(a)$};
\node[draw,rectangle,rounded corners] (sb) at (3,0) {$\mathsf{smokes}(b)$};
\node[draw,rectangle,rounded corners] (iaa) at (0,-1) {$\mathsf{influences}(a,a)$};
\node[draw,rectangle,rounded corners] (ibb) at (3,-1) {$\mathsf{influences}(b,b)$};

\draw[->,>=latex] (iba) to    (sa);
\draw[->,>=latex] (iab) to   (sb);
\draw[->,>=latex] (tb) to   (sb);
\draw[->,>=latex] (sa)  [out=-10, in=-170] to (sb);
\draw[->,>=latex] (sb) [out=170, in=10] to  (sa);  

\draw[->,>=latex] (iaa) to    (sa);
\draw[->,>=latex] (ibb) to   (sb);
\draw[->,>=latex] (ta) to   (sa);

\draw[->,>=latex] (sa) edge [loop left] (sa);
\draw[->,>=latex] (sb) edge [loop right] (sb);
\end{tikzpicture}
\caption{Grounded dependency graph for Example \ref{example:logicSmokers}.}
\label{fig:depSmokers}
\end{figure}
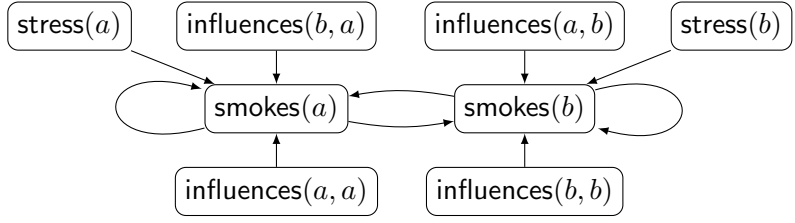

The {\em dependency graph} of a program is a directed graph where each
predicate is a node, and where there is an edge from a node $B$ to a
node $A$ if there is a rule where $A$ appears in the head and $B$
appears in the body; if $B$ appears right after $\mathbf{not}$, the edge
is {\em negative}; otherwise, it is {\em positive}. The {\em grounded}
dependency graph is the dependency graph of the propositional program
obtained by grounding. For instance, the grounded dependency graph of the
program in Example \ref{example:logicSmokers} is depicted in Figure
\ref{fig:depSmokers}.  

A program is {\em acyclic} when its grounded dependency
graph is acyclic. 

Concerning the semantics of normal logic programs, there are, broadly speaking, 
two strategies to follow. One strategy is to translate  programs into a first-order
theory that is called a {\em completion} of the program. Then the semantics of 
the program is the set of first-order models of its completion. The most famous 
completion is Clark's \cite{Clark78}, roughly defined as follows. 
First, rewrite each body by replacing commas by $\wedge$ and $\mathbf{not}$ by $\neg$.
Second, remove constants from heads: to do so, consider a rule $A_0(a) \colonminus B_i\bigperiod$, where 
$a$ is a constant and $B_i$ is the body; then this rule is replaced by 
$A_0(X) \colonminus (X = a) \wedge B_i\bigperiod$. Then, for each 
set of rules that share the same head $A_0$, write
$A_0 \Leftrightarrow B_1 \vee B_2 \vee \dots \vee B_k$, where each $B_i$ is 
the body of one of the rules.  

The second strategy that is often used to define the semantics of normal logic
programs is to select some models of the program to be its semantics. 
There are many proposals in the literature as to which models should be 
selected; however, currently there are two selections that have received 
most attention: the {\em stable model} \cite{Gelfond88} and the {\em well-founded} 
\cite{Gelder91} semantics. We now describe these semantics; alas,
their definitions are not simple.

Consider first the stable model semantics. Suppose we have a normal logic program 
$\mathbf{P}$ and an interpretation $\mathcal{I}$. Define the {\em reduct} 
$\mathbf{P}^\mathcal{I}$ to be a definite program that contains rule 
$A_0 \colonminus A_1,\dots,A_m\bigperiod$ 
iff one of the grounded rules from $\mathbf{P}$ is
$A_0 \colonminus A_1,\dots,A_m,\mathbf{not}\;A_{m+1},\dots,\mathbf{not}\;A_n\bigperiod$ 
 where each $A_{m+1},\dots,A_n$ is  $\mathsf{false}$ in $\mathcal{I}$. 
That is, the reduct is obtained by
(i) grounding $\mathbf{P}$, 
(ii) removing all rules that contain a subgoal $\mathbf{not}\;A$ in their body such that $A$
is an atom that is $\mathsf{true}$ in $\mathcal{I}$,
(iii) removing all remaining literals of the form $\mathbf{not}\;A$ from the remaining rules. 
An interpretation $\mathcal{I}$ is a {\em stable model} if $\mathcal{I}$ is a minimal model
of $\mathbf{P}^\mathcal{I}$. Note that a normal program may fail to have a stable model, or may have
several stable models.

There are two types of logical reasoning under the stable mode
semantics~\cite{Eiter2007}. \emph{Brave reasoning} asks whether there is
a stable model containing a specific atom (and possibly returns it if it
exists). \emph{Cautious reasoning} asks whether a specific atom appears
in all stable models (and possibly lists all such models).

Consider now the well-founded semantics.
Given a subset $\mathcal{U}$ of the Herbrand base of a program, and a partial
interpretation $\mathcal{I}$, say that an 
atom $A$ is {\em unfounded} with respect to $\mathcal{U}$ and $\mathcal{I}$ iff
for each grounded rule whose head is $A$, we have that
(i) some subgoal $A_i$ or $\mathbf{not}\;A_i$ is $\mathsf{false}$ in $\mathcal{I}$, or
(ii) some subgoal that is an atom $A_i$ is in $\mathcal{U}$.
Now say that a subset $\mathcal{U}$ of the Herbrand base is an {\em unfounded set}
with respect to interpretation $\mathcal{I}$ if each atom in $\mathcal{U}$ is unfounded
with respect to $\mathcal{U}$ and $\mathcal{I}$.  
This is a complex definition: roughly, it means that, for each possible rule that we
might apply to obtain $A$, either the rule cannot be used (given $\mathcal{I}$), or 
there is an atom in $\mathcal{U}$ that must be first shown to be $\mathsf{true}$.
%%%%%
Now, given normal logic program $\mathbf{P}$, define $\mathbb{T}_\mathbf{P}(\mathcal{I})$
to be a transformation that takes interpretation $\mathcal{I}$ and returns another interpretation: 
$A \in \mathbb{T}_\mathbf{P}(\mathcal{I})$ iff there is some grounded rule with
head $A$ such that every subgoal in the body is $\mathsf{true}$ in $\mathcal{I}$. 
Also define $\mathbb{U}_\mathbf{P}(\mathcal{I})$ to be the greatest unfounded set with
respect to $\mathcal{I}$ (there is always such a greatest set). 
Define 
$\mathbb{W}_\mathbf{P}(\mathcal{I}) = \mathbb{T}_\mathbf{P}(\mathcal{I}) 
				\cup \neg \mathbb{U}_\mathbf{P}(\mathcal{I})$, 
where the notation $\neg \mathbb{U}_\mathbf{P}(\mathcal{I})$ means that we take each
literal in $\mathbb{U}_\mathbf{P}(\mathcal{I})$ and negate it (that is, $A$ becomes $\neg A$;
 $\neg A$ becomes $A$). Intuitively, $\mathbb{T}_\mathbf{P}$ is what we can
``easily prove to be positive'' and $\mathbb{U}_\mathbf{P}$ is what we can 
``easily prove to be negative''. 

Finally: the well-founded semantics of $\mathbf{P}$ is the least fixed point of
$\mathbb{W}_\mathbf{P}(\mathcal{I})$; this fixed point always exists. That is, apply 
$\mathcal{I}_{i+1} = \mathbb{W}_\mathbf{P}(\mathcal{I}_i)$, starting from 
$\mathcal{I}_0=\emptyset$, until it stabilizes; the resulting interpretation is the
well-founded model. The iteration stops in finitely many steps given that we
have finite Herbrand bases. 

The well-founded semantics   determines the truth assignment for a subset 
of the atoms in the Herbrand base; for the remaining atoms, their ``truth values 
are not determined by the program'' \cite[Section 1.3]{Gelder91}. A very common 
interpretation of this situation is that the well-founded semantics uses three-valued 
logic with values   $\mathsf{true}$, $\mathsf{false}$, and $\mathsf{undefined}$. 
It so happens that any well-founded model is a subset of every stable model of a normal
logic program \cite[Corollary 5.7]{Gelder91}; hence, if a program has a well-founded model
that is an  interpretation for all atoms, then this well-founded model is the unique stable model
(the converse is not true).  

There are other ways to define the well-founded semantics that are explicitly constructive
\cite{Baral93,Gelder93,Przymusinski89}.
One is this, where the connection with the stable model semantics is
emphasized \cite{Baral93}: write $\mathbb{LFT}_\mathbf{P}(\mathcal{I})$ to mean the
least fixpoint of $\mathbb{T}_{\mathbf{P}^\mathcal{I}}$; 
then the well-founded semantics of $\mathbf{P}$
consists of those atoms $A$ that are in the least fixpoint of
$\mathbb{LFT}_\mathbf{P}(\mathbb{LFT}_\mathbf{P}(\cdot))$ plus the literals $\neg A$ for 
those atoms $A$ that are {\em not} in the greatest fixpoint of 
$\mathbb{LFT}_\mathbf{P}(\mathbb{LFT}_\mathbf{P}(\cdot))$. Note that
$\mathbb{LFT}_\mathbf{P}(\mathbb{LFT}_\mathbf{P}(\cdot))$  is a monotone operator.
%%%
% Another definition for the well-founded semantics is the {\em alternating fixpoint
% construction} \cite{Gelder93}, where the well-founded semantics is given by the
% least and greatest fixpoints of the operator $\mathbb{F}^c_\mathbf{P}(\mathbb{F}_\mathbf{P}(\mathcal{I}^c))$,
% where the $\mathcal{I}^c$ denotes the complement of $\mathcal{I}$ with respect to the Herbrand base.

It is instructive to look at some examples.
 
\begin{Example}\label{example:Basic}
First, take a program $\mathbf{P}$ with two rules: 
$\mathsf{p} \colonminus \mathbf{not}\;\mathsf{q}, \mathbf{not}\;\mathsf{r}\bigperiod$ 
and 
$\mathsf{q} \colonminus \mathbf{not}\;\mathsf{p}\bigperiod$ 
(identical to rules in Expression (\ref{equation:BasicCyclicProgram})).
This program has two stable models: both assign $\mathsf{false}$ to $\mathsf{r}$;
one assigns $\mathsf{true}$ to $\mathsf{p}$ and $\mathsf{false}$ to $\mathsf{q}$,
while the other assigns $\mathsf{true}$ to $\mathsf{q}$ and $\mathsf{false}$ to $\mathsf{p}$
(note 
$\mathbf{P}^{\{\mathsf{p},\neg\mathsf{q},\neg\mathsf{r}\}} = \{ \mathsf{p}. \}$ and 
$\mathbf{P}^{\{\neg\mathsf{p},\mathsf{q},\neg\mathsf{r}\}} = \{ \mathsf{q}. \}$).
The well-founded semantics assigns $\mathsf{false}$ to $\mathsf{r}$ and
leaves $\mathsf{p}$ and $\mathsf{q}$ as $\mathsf{undefined}$. 
$\Box$
\end{Example}

\begin{Example}\label{example:Game}
Consider a game where a
player wins if there is another player with no more moves \cite{Gelder91,Gelder93}, as expressed
by the cyclic rule: 
\[
\mathsf{wins}(X) \colonminus  \mathsf{move}(X,Y), \mathbf{not}\;\mathsf{wins}(Y)\bigperiod
\]
Suppose the available moves are given as the following facts:
\[
\mathsf{move}(a,b)\bigperiod \quad \mathsf{move}(b,a)\bigperiod 
\quad \mathsf{move}(b,c)\bigperiod \quad \mathsf{move}(c,d)\bigperiod
\]
There are two stable models: both assign $\mathsf{true}$ to $\mathsf{wins}(c)$ and $\mathsf{false}$ to $\mathsf{wins}(d)$;
one assigns $\mathsf{true}$ to $\mathsf{wins}(a)$ and $\mathsf{false}$ to $\mathsf{wins}(b)$,
while the other assigns $\mathsf{true}$ to $\mathsf{wins}(b)$ and $\mathsf{false}$ to $\mathsf{wins}(a)$. 
The well-founded semantics leads to partial interpretation $\{\mathsf{wins}(c), \neg \mathsf{wins}(d)\}$, 
leaving $\mathsf{wins}(a)$ and $\mathsf{wins}(b)$ as $\mathsf{undefined}$. 
 If $\mathsf{move}(a,b)$ is not given as a fact, it is assigned $\mathsf{false}$, and the
well-founded semantics leads to
$\{\neg\mathsf{wins}(a),\mathsf{wins}(b),\mathsf{wins}(c),\neg\mathsf{wins}(d)\}$. 
$\Box$
\end{Example} 

\begin{Example}\label{example:Barber}
The Barber Paradox: If  the barber shaves all,  and only, those villagers who do not shave themselves,
does the barber shave himself?  Consider:
\begin{equation}
\label{equation:Barber}
\hspace*{-2.48ex}\begin{array}{c}
\mathsf{shaves}(X,Y) \colonminus  \mathsf{barber}(X),  \mathsf{villager}(Y),  \mathbf{not} \; \mathsf{shaves}(Y,Y)\bigperiod \\
\mathsf{villager}(a)\bigperiod \quad
\mathsf{barber}(b)\bigperiod \quad
\mathsf{villager}(b)\bigperiod
\end{array}
\end{equation}
There is no stable model for this normal logic program: the facts and the rule
lead to the pattern $\mathsf{shaves}(b,b) \colonminus \mathbf{not}\;\mathsf{shaves}(b,b)\bigperiod$,
thus eliminating any possible stable model. 
The well-founded semantics assigns $\mathsf{false}$ to $\mathsf{barber}(a)$, 
to $\mathsf{shaves}(a,a)$ and to $\mathsf{shaves}(a,b)$. Also, $\mathsf{shaves}(b,a)$
is assigned $\mathsf{true}$, and $\mathsf{shaves}(b,b)$ is left $\mathsf{undefined}$.
That is, even though the semantics leaves the status of the barber as $\mathsf{undefined}$, it does
produce meaningful answers for other villagers.
\hfill $\Box$
\end{Example}

\subsection{Complexity theory: the counting hierarchy}\label{subsection:Complexity}

We adopt basic terminology and notation from computational
complexity~\cite{Papadimitriou94}. 
A {\em language}  is a set of strings. A language defines a {\em decision problem};
that is, the problem of deciding whether an input string is in the language.
A {\em complexity class} is a set of languages;
we use well-known complexity classes such as  $\mathsf{P}$, $\mathsf{NP}$, 
$\mathsf{EXP}$, $\mathsf{NEXP}$.
The complexity class $\mathsf{PP}$ consists of those languages $\mathcal{L}$
that satisfy the following property: there is a polynomial time 
nondeterministic Turing machine $M$ such that $\ell \in \mathcal{L}$ iff more than half
of the computations of $M$ on input $\ell$ end up accepting. Analogously, we 
have $\mathsf{PEXP}$, consisting of those languages $\mathcal{L}$ with the following property:
 there is 
an exponential time  nondeterministic Turing machine $M$ such that $\ell \in \mathcal{L}$ 
iff half of the computations of $M$ on input $\ell$ end up accepting~\cite{Buhrman98}. 

An oracle Turing machine $M^\mathcal{L}$, where $\mathcal{L}$ is 
a language, is a Turing machine  that  can write a string $\ell$ to an 
``oracle'' tape and obtain from the oracle, in unit time,
the decision as to whether $\ell \in \mathcal{L}$ or not.
Similarly, for a function $f$, an oracle Turing machine $M^f$ can be defined.
 If a class of languages/functions 
$\mathsf{A}$ is defined by a set of Turing machines $\mathcal{M}$ (that is, the 
languages/functions are decided/computed by these machines), then   
$\mathsf{A}^\mathcal{L}$ is the set of languages/functions that 
are decided/computed by $\{ M^\mathcal{L} : M \in \mathcal{M} \}$. 
Similarly, for any class $\mathsf{A}$ we have $\mathsf{A}^f$. 
If $\mathsf{A}$ 
and $\mathsf{B}$ are classes of languages/functions, 
$\mathsf{A}^\mathsf{B} = \cup_{x \in \mathsf{B}} \mathsf{A}^{x}$. 
%If a single call to the oracle is allowed, then write $\mathsf{A}^{\mathsf{B}[1]}$. 
The {\em polynomial hierarchy} consists of classes
$\Pi^\mathsf{P}_i = \mathsf{co}\Sigma^\mathsf{P}_i$ and 
$\Sigma^\mathsf{P}_i = \mathsf{NP}^{\Sigma^{\mathsf{P}}_{i-1}}$, with
$\Sigma^\mathsf{P}_0 = \mathsf{P}$.
Later we also use classes $\Delta^\mathsf{P}_i = \mathsf{P}^{\Sigma^{\mathsf{P}}_{i-1}}$
and  $\mathsf{PH} = \cup_i \Pi^\mathsf{P}_i = \cup_i \Sigma^\mathsf{P}_i$.

Wagner's {\em polynomial counting hierarchy} is the smallest set of classes
containing $\mathsf{P}$ and, recursively, for any class $\mathsf{C}$   in the 
polynomial counting hierarchy, the classes
$\mathsf{PP}^\mathsf{C}$,
$\mathsf{NP}^\mathsf{C}$,
and 
$\mathsf{coNP}^\mathsf{C}$  
\cite[Theorem 4]{Wagner86} \cite[Theorem 4.1]{Toran91}.
The polynomial hierarchy is included in Wagner's counting polynomial hierarchy.
%Note also that $\mathsf{NP}^{\mathsf{NP}^\mathsf{C}} = \mathsf{NP}^{\mathsf{coNP}^\mathsf{C}}$.
%
%A distinct counting hierarchy is Valiant's \cite{Valiant79TCS}.
%We examine this hierarchy further in Section \ref{section:Valiant}; for now suffice to say that
%$\#\mathsf{P}$ is the class of functions such that $f \in \#\mathsf{P}$
%iff $f(\ell)$ is the number of computation paths that accept $\ell$ for some
%polynomial time nondeterministic Turing machine. It is as if we had a
%special machine, called by Valiant a {\em counting} Turing machine, that on
%input $\ell$ prints on a special tape the number of computations that
%accept $\ell$.  
%
%We will also use the class $\mathsf{PP}_1$, defined as the set of languages in $\mathsf{PP}$
%that have a single symbol as input vocabulary. We can take this symbol to be $1$, so
%the input is just a sequence of $1$s (one can interpret this input as a non-negative
%integer written in unary notation). This is the counterpart of Valiant's class $\#\mathsf{P}_1$
%that consists of the functions in $\#\mathsf{P}$ that have a single symbol as input vocabulary
%\cite{Valiant79}.

A {\em many-one reduction} from $\mathcal{L}$ to $\mathcal{L}'$ is a
polynomial time algorithm  that takes the input to decision problem $\mathcal{L}$
and transforms it into the input to decision problem $\mathcal{L}'$ such that 
$\mathcal{L}'$ has the same output as $\mathcal{L}$. 
%A {\em Turing reduction}
%from $\mathcal{L}$ to $\mathcal{L}'$  is an polynomial time algorithm that decides 
%$\mathcal{L}$ using $\mathcal{L}'$ as an oracle. 
%Often such a reduction
%from $\mathcal{L}$ to $\mathcal{L}'$ is indicated by $\mathcal{L} \leq \mathcal{L}'$.
For a complexity class $\mathsf{C}$, a decision problem
$\mathcal{L}$ is $\mathsf{C}$-hard with respect to many-one reductions 
if each decision problem in $\mathsf{C}$ can be reduced to $\mathcal{L}$ with 
many-one reductions.
A decision problem is then $\mathsf{C}$-complete with respect to many-one reductions
 if it is in $\mathsf{C}$
and it is $\mathsf{C}$-hard with respect to many-one reductions. 
%Similar definitions  are obtained when ``many-one reductions''
%are replaced by ``Turing reductions''. 
%%%
%Wagner has presented complete problems, with respect to many-one reductions,
%for all classes in his polynomial counting hierarchy \cite[Theorem 7]{Wagner86}.

In proofs we will often use propositional formulas; such a formula is in
Conjunctive Normal Form (CNF) when it is a conjunction of clauses (where
a clause is a disjunction of literals). A $k$CNF is a CNF in which each clause
has $k$ literals. 
%We denote by $\#\phi$ the number of satisfying assignments
%for a propositional formula $\phi$. 
We use the following
$\mathsf{PP}^{\Sigma_k^\mathsf{P}}$-complete problem \cite[Theorem 7]{Wagner86},
that we refer to as $\#_k\mathsf{3CNF}(>)$:
\begin{description}
\item[Input:] A pair $(\phi,M)$, where $\phi(\mathbf{X}_0,\mathbf{X}_1,\dots,\mathbf{X}_k)$
is a propositional formula in 3CNF  and each $\mathbf{X}_i$ is a tuple of logical variables,
and $M$ is an integer. 
\item[Output:] Whether or not the number of truth assignments for $\mathbf{X}_0$ in the formula
\[
Q_1  \mathbf{X}_1 : Q_2  \mathbf{X}_2 : \dots \exists \mathbf{X}_k : \phi(\mathbf{X}_0,\mathbf{X}_1,\dots,\mathbf{X}_k),
\]
is strictly larger than $M$, where the quantifiers alternate 
and each logical variable not in $\mathbf{X}_0$ is bound to a quantifier.
\end{description}
Another $\mathsf{PP}^{\Sigma_k^\mathsf{P}}$-complete problem,
referred to as $\#_k\mathsf{DNF}(>)$ is:
\begin{description}
\item[Input:] A pair $(\phi,M)$, where $\phi(\mathbf{X}_0,\mathbf{X}_1,\dots,\mathbf{X}_k)$
is a propositional formula in DNF  and each $\mathbf{X}_i$ is a tuple of logical variables,
and $M$ is an integer. 
\item[Output:] Whether or not the number of truth assignments for $\mathbf{X}_0$ in the formula
\[
Q_1  \mathbf{X}_1 : Q_2  \mathbf{X}_2 : \dots \forall \mathbf{X}_k : \phi(\mathbf{X}_0,\mathbf{X}_1,\dots,\mathbf{X}_k),
\]
is strictly larger than $M$, where the quantifiers alternate 
and each logical variable not in $\mathbf{X}_0$ is bound to a quantifier.
\end{description}

A detail is that 
Wagner defines a $\mathsf{PP}^{\Sigma_k^\mathsf{P}}$-complete problem
using ``$\geq k$'' instead of ``$>M$'', but the former is equivalent to ``$>M-1$'', so both
inequalities can be used.  

\section{Probabilistic normal logic programs}\label{section:PLPs}
 
In this paper we focus on a particularly simple combination of logic programming and
probabilities \cite{Poole93AI,Sato95}. A {\em probabilistic logic program}, abbreviated
{\sc plp}, is a pair $\left<\mathbf{P},\mathbf{PF}\right>$ consisting of a normal
logic program $\mathbf{P}$ and a set of {\em probabilistic facts} $\mathbf{PF}$. 
A probabilistic fact is a pair consisting of an atom $A$ and a probability value $\alpha$;
we use the notation $\alpha::A\bigperiod$ borrowed from the ProbLog 
package\footnote{At https://dtai.cs.kuleuven.be/problog/index.html.}
 \cite{Fierens2015}. 

We assume that every probability value is a rational number.

\begin{Example}\label{example:Alarm}
Here is a syntactically correct ProbLog program:
\vspace*{-1ex}
\[
\begin{array}{l}
0.7::\mathsf{burglary}. \qquad 0.2::\mathsf{earthquake}\bigperiod \\
\mathsf{alarm} \colonminus \mathsf{burglary}, \mathsf{earthquake}, \mathsf{a1}\bigperiod  \\
\mathsf{alarm} \colonminus \mathsf{burglary}, \mathbf{not}\ \mathsf{earthquake}, \mathsf{a2}\bigperiod \\
\mathsf{alarm} \colonminus  \mathbf{not}\ \mathsf{burglary}, \mathsf{earthquake}, \mathsf{a3}\bigperiod  \\
  0.9::\mathsf{a1}\bigperiod \quad  0.8::\mathsf{a2}\bigperiod \quad 0.1::\mathsf{a3}\bigperiod \\
\mathsf{calls}(X) \colonminus \mathsf{alarm}, \mathsf{neighbor}(X)\bigperiod \\
\mathsf{neighbor}(a)\bigperiod \qquad \mathsf{neighbor}(b)\bigperiod
\end{array}
\]
There are four rules, two facts, and five probabilistic facts.
%%%% PROBLOG PROGRAM:
%0.7::burglary.
%0.2::earthquake.
%alarm :- burglary, earthquake, a1.
%alarm :- burglary, not earthquake, a2.
%alarm :- not burglary, earthquake, a3.
%0.9::a1.
%0.8::a2.
%0.1::a3.
%calls(X) :- alarm, neighbor(X).
%neighbor(a).
%neighbor(b).
%query(alarm).
%query(calls(a)).
\hfill $\Box$ 
\end{Example} 

A probabilistic fact may contain logical variables; for instance, we may write
$\alpha::\mathsf{r}(X_1,\dots,X_n)\bigperiod$. We interpret such a parameterized probabilistic
fact as the set of all grounded probabilistic facts obtained by substituting variables with 
constants in the Herbrand base.\footnote{ProbLog additionally has ``probabilistic rules''
  but those are simply syntactic sugar that we do not need here.} 

Given a {\sc plp} $\left<\mathbf{P},\mathbf{PF}\right>$ where $\mathbf{P}$ is
acyclic, we say the {\sc plp} is acyclic. Likewise, if $\mathbf{P}$ is definite, 
stratified, cyclic, etc, we use the same adjective for the whole {\sc plp}.

\subsection{The semantics of probabilistic facts}

The interpretation of probabilistic facts requires some pause.
Suppose we have a {\sc plp} $\left<\mathbf{P},\mathbf{PF}\right>$ with $n$ probabilistic facts (which 
may be groundings of probabilistic facts containing logical variables).  From 
$\left<\mathbf{P},\mathbf{PF}\right>$ we can generate $2^n$ normal logic programs: for each probabilistic 
fact $\alpha::A\bigperiod$, we can either choose to keep fact $A\bigperiod$, or choose to erase fact 
$A\bigperiod$ altogether.  
These choices are assumed independent: this is Sato's {\em independence assumption}.

For instance, consider the {\sc plp}:
\begin{equation}
\label{equation:IndependenceAssumption}
0.5::\mathsf{r}\bigperiod \qquad 0.5::\mathsf{s}\bigperiod \qquad \mathsf{v} \colonminus \mathsf{r},\mathsf{s}\bigperiod
\end{equation}
We have four ways to write a normal logic program out of this {\sc plp}; that is,
$\mathsf{r}$ can be kept or removed, and likewise for $\mathsf{s}$. 
All these normal logic programs are obtained with the same probability $0.25$,
and in one of them $\mathsf{v}$ is $\mathsf{true}$; consequently,
the probability $\pr{\mathsf{v}=\mathsf{true}}=0.25$. 

A {\em total choice} $\theta$ for the {\sc plp} is a subset of the set of grounded probabilistic facts.  
We interpret $\theta$ as a set of ground facts that are probabilistic selected to be 
included in $\mathbf{P}$; all other ground  facts obtained from probabilistic facts
are to be discarded.  The probability of a total
choice is easily computed: it is a product over the grounded probabilistic facts,
where probabilistic fact $\alpha::A\bigperiod$ contributes with factor $\alpha$ if $A\bigperiod$ is
kept, or factor $(1-\alpha)$ if $A\bigperiod$ is removed. 
Now for each total choice $\theta$ we obtain a normal logic program,  that we
denote by $\mathbf{P}\cup\mathbf{PF}^{\downarrow\theta}$.

For instance, the {\sc plp} in Expression (\ref{equation:IndependenceAssumption}) has
two probabilistic facts, leading to four total choices, each with probability $0.25$.
Now consider a more complicated {\sc plp}:
\[
0.5::\mathsf{r}\bigperiod \qquad 0.6::\mathsf{r}\bigperiod \qquad 0.2::\mathsf{s}(a)\bigperiod \qquad 
0.3::\mathsf{s}(X)\bigperiod \qquad 
\mathsf{v} \colonminus \mathsf{r}, \mathsf{s}(a), \mathsf{s}(b)\bigperiod
\]
There are five ground probabilistic facts (after grounding $\mathsf{s}(X)$ appropriately);
hence there are $32$ total choices. Suppose we choose to keep the fact in the first
probabilistic fact, and discard all the others (with probability $0.5 \times 0.4 \times 0.8 \times 0.7 \times 0.7$);
then we obtain 
\[
 \mathsf{r}\bigperiod  \qquad 
\mathsf{v} \colonminus \mathsf{r}, \mathsf{s}(a), \mathsf{s}(b)\bigperiod,
\]
a program with a single stable model where $\mathsf{r}$ is the only $\mathsf{true}$ atom.
By going through all possible total choices, we 
have that $\pr{\mathsf{r}=\mathsf{true}}=0.8$
(as $\mathsf{r}\bigperiod$  
is kept in the program by a first choice with probability $0.5$ or by a second choice 
with probability $0.6$, hence $0.5+0.6-0.5 \times 0.6=0.8$). Similarly, 
$\pr{\mathsf{s}(a)=\mathsf{true}}=0.2+0.3-0.2 \times 0.3=0.44$;
note however that $\pr{\mathsf{s}(b)=\mathsf{true}}=0.3$. 
And finally, $\pr{\mathsf{v}=\mathsf{true}}=0.8 \times 0.44 \times 0.3 = 0.1056$. 

Sato assumes that no probabilistic fact unifies with 
the head of a non-fact rule (that is, a rule with a nonempty body); this is
called the {\em disjointness condition}~\cite{Sato95}. From a modeling
perspective this is a convenient assumption even though we do not need
it in our complexity results. In fact from a modeling perspective an even stronger
disjointness condition makes sense: no probabilistic fact should unify with the   
head of any rule (with a body or not), nor with any other probabilistic fact. 
Under this assumption,
 the probabilistic fact $\alpha::A\bigperiod$ can be directly interpreted as a 
probabilistic assessment $\pr{A=\mathsf{true}}=\alpha$. 
Again, we do not need such an assumption for our results, but our examples
will always satisfy it, and it makes sense to assume that it will always be
adopted in practice.

\subsection{The semantics of definite/acyclic/stratified probabilistic logic programs}

We can now discuss the semantics of a {\sc plp}  $\left<\mathbf{P},\mathbf{PF}\right>$.
First, take the grounding of this {\sc plp}.  Now for each total choice $\theta$ we obtain 
the normal logic program $\mathbf{P}\cup\mathbf{PF}^{\downarrow\theta}$. Hence 
the distribution over total choices induces a distribution over normal logic programs.
 
A common assumption   is  that, for each total choice $\theta$, the resulting
normal logic  program $\mathbf{P}\cup\mathbf{PF}^{\downarrow\theta}$ yields a 
single model \cite{Fierens2015}. For instance, if $\mathbf{P}$ is definite, then 
$\mathbf{P}\cup\mathbf{PF}^{\downarrow\theta}$ is definite for any $\theta$, and
$\mathbf{P}\cup\mathbf{PF}^{\downarrow\theta}$   has a unique stable
model that is also its unique well-founded model. Thus the unique distribution over
total choices becomes a unique  distribution over stable/well-founded models. 
This distribution is exactly Sato's {\em distribution semantics}~\cite{Sato95}.
Similarly, suppose that $\mathbf{P}$ is acyclic; then 
$\mathbf{P}\cup\mathbf{PF}^{\downarrow\theta}$ is acyclic for any $\theta$, and
$\mathbf{P}\cup\mathbf{PF}^{\downarrow\theta}$  has a unique stable
model that is also its unique well-founded model~\cite{Apt91}. 

Poole's and Sato's original work focused respectively on acyclic and definite programs;
in both cases the semantics of resulting normal logic programs is uncontroversial. 
The same can be said of the larger class of {\em stratified} programs;  
a normal logic program is  stratified  when cycles 
in the grounded dependency graph contain no negative edge (this is often
referred to as {\em locally stratified} in the literature) \cite{Apt88}.
Both the stable and the well-founded semantics are identical for stratified
programs, and both generate a unique  interpretation for all atoms.
As a consequence, a {\sc plp} $\left<\mathbf{P},\mathbf{PF}\right>$ has a unique distribution semantics
whenever $\mathbf{P}$ is stratified. Note that both acyclic and definite programs are stratified. 

\begin{Example}
The {\sc plp} in Example \ref{example:Alarm} is acyclic, and thus stratified, but not
definite. The grounded dependency graph of this program is depicted in Figure~\ref{figure:AlarmGraph}. 
This graph can be interpreted as a Bayesian network, as we discuss later \cite{Poole93AI}. 
There are $2^5$ total choices, and the probability of $\mathsf{calls}(\mathsf{a})$ is $0.58$.
\end{Example}

\begin{figure}
\begin{center}
\begin{tikzpicture}[scale=1.3]
\node[draw,rectangle,rounded corners] (b) at (1,0.6) {$\mathsf{burglary}$};
\node[draw,rectangle,rounded corners] (e) at (1,-0.1) {$\mathsf{earthquake}$};
\node[draw,rectangle,rounded corners] (a) at (3.5,0.6) {$\mathsf{alarm}$};
\node[draw,rectangle,rounded corners] (cM) at (6,0.62) {$\mathsf{calls}(a)$};
\node[draw,rectangle,rounded corners] (cJ) at (6,-0.1) {$\mathsf{calls}(b)$};
\node[draw,rectangle,rounded corners] (nM) at (8,0.62) {$\mathsf{neighbor}(a)$};
\node[draw,rectangle,rounded corners] (nJ) at (8,-0.1) {$\mathsf{neighbor}(b)$};
\node[draw,rectangle,rounded corners] (a1) at (2.5,-0.1) {$\mathsf{a1}$};
\node[draw,rectangle,rounded corners] (a2) at (3.5,-0.1) {$\mathsf{a2}$};
\node[draw,rectangle,rounded corners] (a3) at (4.5,-0.1) {$\mathsf{a3}$};
\draw[->,>=latex] (b)--(a);
\draw[->,>=latex] (e)--(a);
\draw[->,>=latex] (a)--(cM);
\draw[->,>=latex] (a)--(cJ);
\draw[->,>=latex] (a1)--(a);
\draw[->,>=latex] (a2)--(a);
\draw[->,>=latex] (a3)--(a);
\draw[->,>=latex] (nM)--(cM);
\draw[->,>=latex] (nJ)--(cJ);
\end{tikzpicture}
\end{center}
\vspace*{-3ex}
\caption{The grounded dependency graph for Example \ref{example:Alarm}.}
\label{figure:AlarmGraph}
\end{figure}
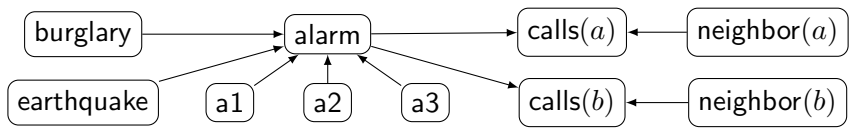

\begin{Example}\label{example:Smokers}
Consider a probabilistic version of the ``smokers'' program in
Example \ref{example:logicSmokers} \cite{Fierens2015}:
  \[
    \begin{array}{l}
      \mathsf{smokes}(X) \colonminus \mathsf{stress}(X)\bigperiod \\
      \mathsf{smokes}(X) \colonminus \mathsf{influences}(Y,X), \mathsf{smokes}(Y)\bigperiod \\
     0.3::\mathsf{influences}(a,b)\bigperiod \;\;  0.3::\mathsf{influences}(b,a).
       \;\;   0.8::\mathsf{stress}(b)\bigperiod
     \end{array}
  \]
The  grounded dependency graph of this program is identical to the one shown in Figure \ref{fig:depSmokers}.
It is tempting to interpret this graph as a Bayesian network, but of course this 
is not quite right as the graph is cyclic. Indeed the program is not acyclic, but it is definite and therefore
stratified, hence a unique distribution is defined over ground atoms.
For instance, we have $\pr{\mathsf{smokes}(a)} = 0.06$ and $\pr{\mathsf{smokes}(b)} = 0.2$.
The program would still be stratified if the first rule were replaced by 
\[
\mathsf{smokes}(X) \colonminus \mathbf{not}\;\mathsf{stress}(X)\bigperiod
\]
In this case there would still
be a cycle, but the negative edge in the dependency graph would not belong to the cycle. 
\hfill $\Box$
\end{Example}
%% IN PROBLOG:
%s(X) :- t(X).
%s(X) :- i(Y,X), s(Y).
%0.3::i(a,b).
%0.3::i(b,a).
%0.2::t(b).
%query(s(a)).
%query(s(b)).
%%%%% 

Often a stratified program is used to implement recursion, as illustrated by the next example:

\begin{Example}
Consider the following {\sc plp}, based on an example in the ProbLog distribution:
\[
\label{equation:TransitiveClosure}
\begin{array}{c}
\mathsf{path}(X,Y) \colonminus \mathsf{edge}(X,Y)\bigperiod  \\
\mathsf{path}(X,Y) \colonminus \mathsf{edge}(X,Y), \mathsf{path}(X,Y)\bigperiod \\
0.6::\mathsf{edge}(\mathsf{1},\mathsf{2})\bigperiod \quad 
0.1::\mathsf{edge}(\mathsf{1},\mathsf{3})\bigperiod \quad 
0.4::\mathsf{edge}(\mathsf{2},\mathsf{5})\bigperiod \quad 
0.3::\mathsf{edge}(\mathsf{2},\mathsf{6})\bigperiod \\
0.3::\mathsf{edge}(\mathsf{3},\mathsf{4})\bigperiod \quad 
0.8::\mathsf{edge}(\mathsf{4},\mathsf{5})\bigperiod \quad 
0.2::\mathsf{edge}(\mathsf{5},\mathsf{6})\bigperiod
\end{array}
\]
That is, we have a random graph with nodes $\mathsf{1},\dots,\mathsf{6}$, and probabilities attached to
  edges.   The query $\pr{\mathsf{path}(\mathsf{1},\mathsf{6})=\mathsf{true}}$ yields
the probability that there is a path between nodes $\mathsf{1}$ and $\mathsf{6}$.
Using ProbLog one obtains $\pr{\mathsf{path}(\mathsf{1},\mathsf{6})=\mathsf{true}}=0.217$. 
$\Box$
\end{Example}

\subsection{The semantics of general probabilistic logic programs}

If a normal logic program is non-stratified, then its well-founded semantics may be
a partial interpretation, and some atoms may be left as $\mathsf{undefined}$; 
it may have several stable models, or no stable model at all.  
%One might focus on programs that have a single stable model per total
%choice \cite{Riguzzi2015}, % Riguzzi refers to these programs as {\em sound} ones 
%but this seems quite restrictive. 
Thus we must accommodate these cases when we contemplate non-stratified
{\sc plp}s.

\subsubsection{The credal semantics}

A first possible semantics for general probabilistic logic programs can be extracted from 
work by \citeA{Lukasiewicz2005ecsqaru,Lukasiewicz2007} on probabilistic description logic programs.
To describe that  proposal, a few definitions are needed. 
A  {\sc plp} $\left< \mathbf{P},\mathbf{PF} \right>$ is {\em consistent}
if there is at least one stable model for each   total choice of $\mathbf{PF}$.  
A {\em probability model} for a consistent {\sc plp}
$\left< \mathbf{P},\mathbf{PF} \right>$ is a probability measure $\pro$ 
over interpretations of $\mathbf{P}$,  such that: \\
(i) every interpretation $\mathcal{I}$ with $\pr{\mathcal{I}}>0$ is a stable 
model of $\mathbf{P} \cup \mathbf{PF}^{\downarrow\theta}$ for the total choice $\theta$ that
agrees with $\mathcal{I}$ on the probabilistic facts (that is, $\theta$ induces the same
truth values as $\mathcal{I}$ for the grounded probabilistic facts); and \\
(ii) the probability of each total choice $\theta$ is the product of the probabilities
for all individual choices in $\theta$. \\
The set of all probability models for a {\sc plp} is the semantics of 
the program.  Later examples will clarify this construction.

Lukasiewicz calls his proposed semantics the {\em answer set semantics}  for
probabilistic description logic programs; however, note that this name is both
too restrictive (the semantics can be used for programs with functions, for instance)
and a bit opaque (it does not emphasize the fact that it deals with uncertainty).
We prefer the term {\em credal semantics}, which we adopt from now on. 
The reason for this latter name is that a set of probability measures is often called 
a {\em credal set}  \cite{Augustin2014}. 

Now given a consistent {\sc plp}, we may be interested
in the smallest possible value of $\pr{\mathbf{Q}}$ for a set $\mathbf{Q}$
of truth assignments,  
with respect to the set $\credalo$ of all probability models of the {\sc plp}.
 This is conveyed by the {\em lower probability} of $\mathbf{Q}$,
$\lpr{\mathbf{Q}} = \inf_{\pro \in \credalo} \pr{\mathbf{Q}}$. Similarly, we have the 
{\em upper probability} of $\mathbf{Q}$, $\upr{\mathbf{Q}} = \sup_{\pro \in \credalo} \pr{\mathbf{Q}}$. 
Suppose that we also have a set of $\mathbf{E}$ of truth assignments for
ground atoms; then we may be interested in the conditional lower and upper
probabilities, respectively
$\lpr{\mathbf{Q}|\mathbf{E}} = \inf_{\pro \in \credalo: \pr{\mathbf{E}}>0} \pr{\mathbf{Q}|\mathbf{E}}$ 
and 
$\upr{\mathbf{Q}|\mathbf{E}} = \sup_{\pro \in \credalo: \pr{\mathbf{E}}>0} \pr{\mathbf{Q}|\mathbf{E}}$.
We leave conditional lower/upper probabilities undefined when $\upr{\mathbf{E}}=0$ (that is, when
$\pr{\mathbf{E}}=0$ for every probability model).
This is not the only possible convention:
\citeA[Section 3]{Lukasiewicz2005ecsqaru} adopts $\lpr{\mathbf{Q}|\mathbf{E}}=1$ and
$\upr{\mathbf{Q}|\mathbf{E}}=0$ in this case, while Walley's style of conditioning 
prescribes $\lpr{\mathbf{Q}|\mathbf{E}}=0$ and
$\upr{\mathbf{Q}|\mathbf{E}}=1$ whenever $\upr{\mathbf{E}}=0$
\cite{Walley91}.

\subsubsection{The well-founded semantics}

The approach by  \citeA{Hadjichristodoulou2012} 
is to allow probabilities directly over well-founded models, thus allowing probabilities
over atoms that are $\mathsf{undefined}$. That is, given a {\sc plp} 
$\left< \mathbf{P},\mathbf{PF} \right>$,  associate to each total choice $\theta$
the unique well-founded model of $\mathbf{P}\cup\mathbf{PF}^{\downarrow\theta}$ to
$\theta$; the unique distribution over total choices induces a unique distribution 
over well-founded models. Note that probabilities may be assigned to $\mathsf{undefined}$
values in this sort of semantics. As we discuss in Section \ref{section:Semantics}, this
is a bold proposal whose interpretation is far from simple.

Regardless of its meaning, the approach deserves attention as it is the only
one in the literature that genuinely combines well-founded semantics  with
probabilities. Accordingly, we refer to it as the {\em well-founded semantics}
of probabilistic logic programs (the combination of language and semantics 
is named WF-PRISM by Hadjichristodoulou and Warren). 

\subsubsection{Other semantics}

Sato et  al.\   propose a semantics where 
distributions are defined over models produced
by  Fitting's three-valued semantics \cite{Sato2005}.
We note that Fitting's semantics is weaker than
the well-founded semantics, and the literature on logic programming has consistently 
preferred the latter, as we do in this paper.

Another three-valued approach, proposed by  \citeA{Lukasiewicz2005ecsqaru,Lukasiewicz2007},
leaves the probability of any formula 
as $\mathsf{undefined}$ whenever the formula is  $\mathsf{undefined}$ for
any total choice (to determine whether a formula is $\mathsf{undefined}$
or not in a particular partial interpretation, three-valued logic is used). 
Hence, when a formula gets a (non-$\mathsf{undefined}$) numeric probability value, 
its truth value is the same for all stable models; thus any numeric probability calculations 
that are produced with this semantics agree with the semantics based on stable
models \citeA[Theorem 4.5]{Lukasiewicz2007}.
That is, Lukasiewicz' proposal is more akin to the credal semantics than
to the well-founded semantics.

A different semantics for non-stratified {\sc plp}s is adopted by the P-log language
\cite{Baral2009}. P-log allows for disjunction in heads and other features, but when 
restricted to normal logic programs it is syntactically similar to ProbLog. The semantics
of a P-log program is given by a single probability distribution over possibly many
stable  models; whenever necessary default assumptions 
are called to distribute probability evenly, or to avoid inconsistent realizations 
(by re-normalization). We leave an analysis of this sort of semantics to the future;
here we prefer to focus on semantics that do not make default assumptions 
concerning probabilities. 

It is also important to mention the constraint logic programming
language of  \citeA{Michels2015}, a significant contribution that 
is also based on credal sets. However, they use a syntax and semantics that
is markedly different from Lukasiewicz's approach, as they allow 
continuous variables but do not let a program have multiple stable
models per total choice. They also present expressions for (conditional) lower and
upper probabilities, by direct optimization; in Section \ref{section:Structure}
we show that such expressions can be derived from properties of 
infinitely monotone Choquet capacities. 

Finally,  \citeA{Ceylan2016} have introduced a semantics that
allows for inconsistent {\sc plp}s to have meaning without getting into
three-valued logic.  They adopt a much more sophisticated family
of logic programs (within the Datalog$^\pm$ language), and they provide
a thorough analysis of complexity that we discuss later. This is also a 
proposal that deserves future study. 

In this paper we focus on the {\em credal} and the {\em well-founded} semantics
in the remainder of this paper, whenever non-stratified {\sc plp}s are discussed, 
but certainly there are other avenues to explore.

\section{The semantics of the credal and the well-founded semantics}
\label{section:Semantics}

It does not seem that any comparison is available in the literature between
the credal and the well-founded semantics for non-stratified {\sc plp}s. 
Indeed, the credal semantics has not been adopted since its appearance,
a turn of events we find unfortunate as it is quite a sensible semantics for
general {\sc plp}s. In this section we present some examples that emphasize
differences between these semantics, and we examine their interpretation. 

\begin{Example}\label{example:ProbabilisticBasic}
Consider a probabilistic version of Example \ref{example:Basic}:
\[
\mathsf{p} \colonminus \mathbf{not}\;\mathsf{q}, \mathbf{not}\;\mathsf{r}\bigperiod
\qquad 
\mathsf{q} \colonminus \mathbf{not}\;\mathsf{p}\bigperiod
\qquad
\alpha::\mathsf{r}\bigperiod
\]
%%%%
%p :- not q, not r.
%q :- not p.
%r.
%%%%
This is in essence identical to the {\sc plp} in
Expression~(\ref{equation:BasicCyclicProgram}). To interpret the {\sc plp}, note that
with probability $\alpha$ we obtain the normal logic program
\[
\mathsf{p} \colonminus \mathbf{not}\;\mathsf{q}, \mathbf{not}\;\mathsf{r}\bigperiod
\qquad 
\mathsf{q} \colonminus \mathbf{not}\;\mathsf{p}\bigperiod
\qquad
\mathsf{r}\bigperiod
\]
The unique stable/well-founded model of this program assigns $\mathsf{true}$ to $\mathsf{r}$ and $\mathsf{q}$,
and $\mathsf{false}$ to $\mathsf{p}$. That is, we have the stable
model $s_1 = \{ \neg \mathsf{p}, \mathsf{q}, \mathsf{r} \}$. 
On the other hand, with probability $1-\alpha$ we
obtain a program with different behavior, namely:
\[
\mathsf{p} \colonminus \mathbf{not}\;\mathsf{q}, \mathbf{not}\;\mathsf{r}\bigperiod
\qquad 
\mathsf{q} \colonminus \mathbf{not}\;\mathsf{p}\bigperiod
\]
This program has two stable models: $s_2 = \{ \mathsf{p}, \neg \mathsf{q}, \neg \mathsf{r} \}$
and $s_3 = \{ \neg \mathsf{p}, \mathsf{q}, \neg \mathsf{r} \}$. 
But this program has a single well-founded model, where $\mathsf{r}$ is $\mathsf{false}$
and both $\mathsf{p}$ and $\mathsf{q}$ are $\mathsf{undefined}$. 

Consider the credal semantics.
There is a probabilty model such that $\pr{s_2}=1-\alpha$ and $\pr{s_3}=0$, 
and another probability model such that $\pr{s_2}=0$ and $\pr{s_3}=1-\alpha$. 
Indeed any probability measure such that 
$\pr{s_1}=\alpha$ and $\pr{s_2}=\gamma(1-\alpha)$, $\pr{s_3}=(1-\gamma)(1-\alpha)$,
for $\gamma \in [0,1]$, is also a probability model for this {\sc plp}. 

The well-founded semantics is instead a single distribution that assigns
$\pr{s_1}=1-\alpha$, and assigns probability mass $\alpha$ to the
partial interpretation $\{\neg \mathsf{r}\}$. 

Now consider an inference; say for instance one wants $\pr{\mathsf{r}=\mathsf{true}}$.
Clearly $\pr{\mathsf{r}=\mathsf{true}}=1-\alpha$, regardless of the semantics.
But consider $\mathsf{p}$. 
With respect to the credal semantics, the relevant quantities are
$\lpr{\mathsf{p}=\mathsf{true}}=0$
and
$\upr{\mathsf{p}=\mathsf{true}}=1-\alpha$. 
%And similarly,
%$\lpr{\mathsf{p}=\mathsf{true}}=\alpha$
%and
%$\upr{\mathsf{p}=\mathsf{true}}=1$. 
And with respect to the well-founded semantics we have instead
$\pr{\mathsf{p}=\mathsf{true}}=0$
and $\pr{\mathsf{p}=\mathsf{false}}=\alpha$,
while $\pr{\mathsf{p}=\mathsf{undefined}}=1-\alpha$. 

To elaborate on this sort of programming pattern, consider the following 
non-propositional example,
adapted from  \citeA{Eiter2009Primer}:
\[
\begin{array}{c}
0.9::\mathsf{man}(\mathsf{dilbert})\bigperiod \\
\mathsf{single}(X) \colonminus \mathsf{man}(X), \mathbf{not}\;\mathsf{husband}(X)\bigperiod \\
\mathsf{husband}(X) \colonminus \mathsf{man}(X), \mathbf{not}\;\mathsf{single}(X)\bigperiod
\end{array}
\]
When $\mathsf{man}(\mathsf{dilbert})$ is discarded, the resulting normal
logic program has a single stable model
$s_1 = \{\neg \mathsf{man}(\mathsf{dilbert}),
              \neg \mathsf{husband}(\mathsf{dilbert}),
              \neg \mathsf{single}(\mathsf{dilbert})\}$.
When $\mathsf{man}(\mathsf{dilbert})$ is a fact,
the resulting program two stable models:
\begin{eqnarray*}
s_2 & = & \{\mathsf{man}(\mathsf{dilbert}),
       \mathsf{husband}(\mathsf{dilbert}),
        \neg  \mathsf{single}(\mathsf{dilbert})\}, \\
s_3 & = &  \{\mathsf{man}(\mathsf{dilbert}), 
        \neg \mathsf{husband}(\mathsf{dilbert}),
         \mathsf{single}(\mathsf{dilbert})\};
\end{eqnarray*}
the well-founded semantics instead leads to $\mathsf{undefined}$
values both for $\mathsf{husband}(\mathsf{dilbert})$ and 
$\mathsf{single}(\mathsf{dilbert})$. 
%Thus there is a probability model  such that  
%$\pr{s_1}=0.1$ and $\pr{s_2}=0.9$ (hence $\pr{s_3}=0$), 
%and a  probability set model such that 
%$\pr{s_1}=0.1$ and $\pr{s_3}=0.9$.

Note that any probability measure such that 
$\pr{s_1}=0.1$, 
$\pr{s_2}=0.9\gamma$,
and 
$\pr{s_3}=0.9(1-\gamma)$, 
for $\gamma\in[0,1]$, is a probability model.
Hence we have 
$\lpr{\mathsf{husband}(\mathsf{dilbert})=\mathsf{true}}=0$
and 
$\upr{\mathsf{husband}(\mathsf{dilbert})=\mathsf{true}}=0.9$
with respect to the credal semantics,
while we have
$\pr{\mathsf{husband}(\mathsf{dilbert})=\mathsf{true}}=0$,
$\pr{\mathsf{husband}(\mathsf{dilbert})=\mathsf{false}}=0.1$,
and finally we have
$\pr{\mathsf{husband}(\mathsf{dilbert})=\mathsf{undefined}}=0.9$
with respect to the well-founded semantics.
$\hfill \Box$
\end{Example}

\begin{Example}\label{example:CyclicProgram}
Now take a {\sc plp} adapted from an example by
   \citeA[Example IV.1]{Hadjichristodoulou2012},
where the same pattern of cyclic negation observed in the previous example seems
to appear: 
\[
\begin{array}{c}
\mathsf{cold} \colonminus \mathsf{headache}, \mathsf{a}\bigperiod \qquad 
\mathsf{cold} \colonminus \mathbf{not}\;\mathsf{headache}, \mathbf{not}\;\mathsf{a}. \qquad
0.34::\mathsf{a}\bigperiod \\
\mathsf{headache} \colonminus \mathsf{cold}, \mathsf{b}\bigperiod \qquad 
\mathsf{headache} \colonminus  \mathbf{not}\;\mathsf{b}\bigperiod \qquad
0.25::\mathsf{b}\bigperiod
\end{array}
\]
There are four total choices, each inducing a normal logic
program. In one case, namely $\{ \neg \mathsf{a},\mathsf{b} \}$, the resulting
normal logic program has no stable model. Hence, this {\sc plp} has no credal semantics.
However, it does have a well-founded semantics. Table \ref{table:CyclicProgram}
shows the assignments for $\mathsf{cold}$ and $\mathsf{headache}$ induced
by the various total choices; we obtain
\[
\begin{array}{ccc}
\pr{\mathsf{cold}=\mathsf{true}}=0.255,          & \hspace*{1cm} &  \pr{\mathsf{headache}=\mathsf{true}}=0.750, \\
\pr{\mathsf{cold}=\mathsf{undefined}}=0.165, & \hspace*{1cm} &  \pr{\mathsf{headache}=\mathsf{undefined}}=0.165, \\
\pr{\mathsf{cold}=\mathsf{false}}=0.580,        & \hspace*{1cm} &   \pr{\mathsf{headache}=\mathsf{false}}=0.085. 
\end{array}
\]
by collecting probabilities from Table \ref{table:CyclicProgram}.
$\hfill \Box$
\end{Example}

\begin{table}
\begin{center}
\begin{tabular}{|c|c|c|c|c|} \hline
$\mathsf{a}$ & $\mathsf{b}$ & $\mathsf{cold}$ & $\mathsf{headache}$ & Probability  \\ \hline \hline
$\mathsf{true}$ & $\mathsf{true}$ & $\mathsf{false}$ & $\mathsf{false}$   
         & $0.34 \times 0.25 = 0.085$  \\ \hline
$\mathsf{true}$ & $\mathsf{false}$ & $\mathsf{true}$ & $\mathsf{true}$ 
          & $0.34 \times 0.75 = 0.255$ \\ \hline
$\mathsf{false}$ & $\mathsf{true}$ & $\mathsf{undefined}$ & $\mathsf{undefined}$ 
          & $0.66 \times 0.25 = 0.165$ \\ \hline
$\mathsf{false}$ & $\mathsf{false}$ & $\mathsf{false}$ & $\mathsf{true}$
          & $0.66 \times 0.75 = 0.495$ \\ \hline
\end{tabular}
\end{center}
\caption{Total choices, the induced assignments, and their probabilities, 
for Example \ref{example:CyclicProgram}.}
\label{table:CyclicProgram}
\end{table}

\begin{Example}\label{example:GraphColoring}
Consider a graph coloring problem consisting of the rules:
\[
\begin{array}{c}
\mathsf{color}(V,\mathsf{red}) \colonminus­ \mathbf{not}\;\mathsf{color}(V,\mathsf{yellow}), \mathbf{not}\;\mathsf{color}(V,\mathsf{green}), \mathsf{vertex}(V)\bigperiod \\
\mathsf{color}(V,\mathsf{yellow}) \colonminus­ \mathbf{not}\;\mathsf{color}(V,\mathsf{red}), \mathbf{not}\;\mathsf{color}(V,\mathsf{green}), \mathsf{vertex}(V)\bigperiod \\
\mathsf{color}(V,\mathsf{green}) \colonminus­ \mathbf{not}\;\mathsf{color}(V,\mathsf{red}), \mathbf{not}\;\mathsf{color}(V,\mathsf{yellow}), \mathsf{vertex}(V)\bigperiod \\
\mathsf{clash} \colonminus \mathbf{not}\;\mathsf{clash}, 
	\mathsf{edge}(V,U), \mathsf{color}(V,C), \mathsf{color}(U,C)\bigperiod
\end{array}
\]
and the facts: 
for $i \in \{\mathsf{1}, \dots, \mathsf{5}\}$, $\mathsf{vertex}(i)\bigperiod$, and
\[
\begin{array}{c}
\mathsf{color}(\mathsf{2},\mathsf{red})\bigperiod \quad 
\mathsf{color}(\mathsf{5},\mathsf{green})\bigperiod \\
0.5::\mathsf{edge}(\mathsf{4},\mathsf{5})\bigperiod \\
\mathsf{edge}(\mathsf{1},\mathsf{3})\bigperiod \quad 
\mathsf{edge}(\mathsf{1},\mathsf{4})\bigperiod \quad
\mathsf{edge}(\mathsf{2},\mathsf{1})\bigperiod \quad 
\mathsf{edge}(\mathsf{2},\mathsf{4})\bigperiod \quad
\mathsf{edge}(\mathsf{3},\mathsf{5})\bigperiod \quad   
\mathsf{edge}(\mathsf{4},\mathsf{3})\bigperiod    
\end{array}
\]
The facts mentioning $\mathsf{vertex}$ and $\mathsf{edge}$
encode the graph in Figure \ref{figure:GraphColoring} (left); the probabilistic fact is
indicated as a dashed edge.
%%
%%%%%%%% HERE IS THE CLINGO CODE:
%color(V,red) :- vertex(V), not color(V,yellow), not color(V,green).
%color(V,yellow) :- vertex(V), not color(V,red), not color(V,green).
%color(V,green) :- vertex(V), not color(V,red), not color(V,yellow).
%noClash :- not noClash, edge(V,U), color(V,C), color(U,C).
%
%vertex(a1).
%vertex(a2).
%vertex(a3).
%vertex(a4).
%vertex(a5).
%edge(a1,a3).
%edge(a1,a4).
%edge(a2,a1).
%edge(a2,a4).
%edge(a3,a5).
%edge(a4,a3).
%edge(a4,a5).
%color(a2,red).
%color(a5,green).
%%%%%%%%%%%%%%%%%%%%%%%%%%%%
%%%
A total choice determines a particular graph. 
For a fixed total choice, the stable models of the program are
the 3-colorings of the resulting graph (this is indeed a popular example
of {\em answer set programming} \cite{Eiter2009Primer}).

Now, if  probabilistic fact 
$\mathsf{edge}(\mathsf{4},\mathsf{5})$ is $\mathsf{true}$,
there is a single stable model; otherwise, there are two stable models.
Using the credal semantics we obtain:
$\lpr{\mathsf{color}(\mathsf{1},\mathsf{yellow})=\mathsf{true}}=0$
and
$\upr{\mathsf{color}(\mathsf{1},\mathsf{yellow})=\mathsf{true}}=1/2$;
also, we have
$\lpr{\mathsf{color}(\mathsf{4},\mathsf{yellow})=\mathsf{true}}=1/2$
and
$\upr{\mathsf{color}(\mathsf{4},\mathsf{yellow})=\mathsf{true}}=1$;
and 
$\lpr{\mathsf{color}(\mathsf{3},\mathsf{red})=\mathsf{true}}=
   \upr{\mathsf{color}(\mathsf{3},\mathsf{red})=\mathsf{true}}=1$.

On the other hand, the well-founded semantics leaves undefined the
colors of vertices $\mathsf{1}$, $\mathsf{3}$, and $\mathsf{4}$, both
when $\mathsf{edge}(\mathsf{4},\mathsf{5})$ is $\mathsf{true}$
and when it is $\mathsf{false}$. 
Thus we have, for
$V \in \{\mathsf{1},\mathsf{3},\mathsf{4}\}$ and 
$C \in \{\mathsf{red},\mathsf{yellow},\mathsf{green}\}$, that 
$\pr{\mathsf{color}(V,C)=\mathsf{undefined}}=1$.
$\hfill \Box$
\end{Example}

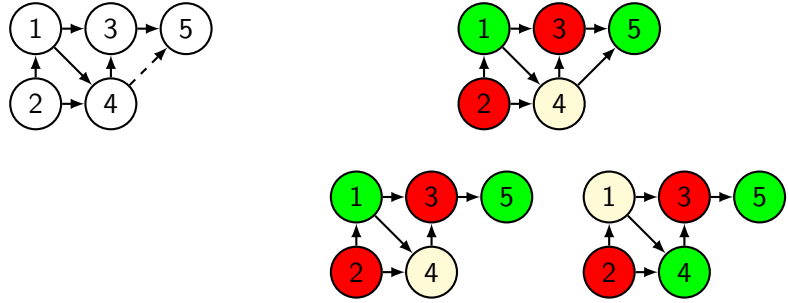
\begin{figure}[t]
\begin{center}
\vspace*{5mm}
\begin{tikzpicture}[thick,->, >=latex]
\node[circle,draw] (v1) at (0,2) {$\mathsf{1}$};
\node[circle,draw] (v2) at (0,1) {$\mathsf{2}$};
%\node[circle,draw,fill=red] (v2) at (0,1) {$\mathsf{2}$};
\node[circle,draw] (v3) at (1,2) {$\mathsf{3}$};
\node[circle,draw] (v4) at (1,1) {$\mathsf{4}$};
\node[circle,draw] (v5) at (2,2) {$\mathsf{5}$};
%\node[circle,draw,fill=green] (v5) at (2,2) {$\mathsf{5}$};
\draw (v1)--(v3);
\draw (v1)--(v4);
\draw (v2)--(v1);
\draw (v2)--(v4);
\draw (v3)--(v5);
\draw (v4)--(v3);
\draw[dashed] (v4)--(v5);
\end{tikzpicture}
\hspace*{30mm}
\begin{tikzpicture}[thick,->, >=latex]
\node[circle,draw,fill=green] (v1) at (0,2) {$\mathsf{1}$};
\node[circle,draw,fill=red] (v2) at (0,1) {$\mathsf{2}$};
\node[circle,draw,fill=red] (v3) at (1,2) {$\mathsf{3}$};
\node[circle,draw,fill=yellow!20] (v4) at (1,1) {$\mathsf{4}$};
\node[circle,draw,fill=green] (v5) at (2,2) {$\mathsf{5}$};
\draw (v1)--(v3);
\draw (v1)--(v4);
\draw (v2)--(v1);
\draw (v2)--(v4);
\draw (v3)--(v5);
\draw (v4)--(v3);
\draw (v4)--(v5);
\end{tikzpicture}

\vspace*{5mm}

\hspace*{58mm}
\begin{tikzpicture}[thick,->, >=latex]
\node[circle,draw,fill=green] (v1) at (0,2) {$\mathsf{1}$};
\node[circle,draw,fill=red] (v2) at (0,1) {$\mathsf{2}$};
\node[circle,draw,fill=red] (v3) at (1,2) {$\mathsf{3}$};
\node[circle,draw,fill=yellow!20] (v4) at (1,1) {$\mathsf{4}$};
\node[circle,draw,fill=green] (v5) at (2,2) {$\mathsf{5}$};
\draw (v1)--(v3);
\draw (v1)--(v4);
\draw (v2)--(v1);
\draw (v2)--(v4);
\draw (v3)--(v5);
\draw (v4)--(v3); 
\end{tikzpicture}
\hspace*{4mm}
\begin{tikzpicture}[thick,->, >=latex]
\node[circle,draw,fill=yellow!20] (v1) at (0,2) {$\mathsf{1}$};
\node[circle,draw,fill=red] (v2) at (0,1) {$\mathsf{2}$};
\node[circle,draw,fill=red] (v3) at (1,2) {$\mathsf{3}$};
\node[circle,draw,fill=green] (v4) at (1,1) {$\mathsf{4}$};
\node[circle,draw,fill=green] (v5) at (2,2) {$\mathsf{5}$};
\draw (v1)--(v3);
\draw (v1)--(v4);
\draw (v2)--(v1);
\draw (v2)--(v4);
\draw (v3)--(v5);
\draw (v4)--(v3); 
\end{tikzpicture}
\end{center}
\vspace*{-3ex}
\caption{Graph described in Example \ref{example:GraphColoring} (left), and the stable models
produced by fixing colors to nodes $\mathsf{2}$ and $\mathsf{5}$: one stable model 
is obtained for a total choice (right, top), and two stable models are obtained for the other
total choice (right, bottom). [Note that ``yellow'' appears as very light grey, ``green'' appears
as light grey, and ``red'' appears
as dark grey when printed in black-and-white.]}
\label{figure:GraphColoring}
\end{figure}

\begin{Example}\label{example:Wins}
Take the normal logic program discussed in Example \ref{example:Game},
and consider the following probabilistic version (there is one probabilistic
move in the game):
\[
\begin{array}{c}
\mathsf{wins}(X) \colonminus \mathsf{move}(X,Y), \mathbf{not}\;\mathsf{wins}(Y)\bigperiod \\
\mathsf{move}(\mathsf{a},\mathsf{b})\bigperiod \quad
\mathsf{move}(\mathsf{b},\mathsf{a})\bigperiod \quad
\mathsf{move}(\mathsf{b},\mathsf{c})\bigperiod \quad
0.3::\mathsf{move}(\mathsf{c},\mathsf{d})\bigperiod
\end{array}
\]

If $\mathsf{move}(\mathsf{c},\mathsf{d})$ is discarded, there is a single
stable model (where $\mathsf{b}$ is the only winning position);
otherwise, there are two stable models ($\mathsf{wins}(\mathsf{c})$ is $\mathsf{true}$
and $\mathsf{wins}(\mathsf{d})$ is $\mathsf{false}$
in both of them; $\mathsf{wins}(\mathsf{a})$ is $\mathsf{true}$ in one, while 
$\mathsf{wins}(\mathsf{b})$ is $\mathsf{true}$ in the other).
Thus the credal semantics yields
%$\lpr{\mathsf{wins}(\mathsf{a})}= 0.0$ and $\upr{\mathsf{wins}(\mathsf{a})} = 0.3$;
$\lpr{\mathsf{wins}(\mathsf{b})=\mathsf{true}}= 0.7$ and 
$\upr{\mathsf{wins}(\mathsf{b})=\mathsf{true}} = 1.0$;
$\lpr{\mathsf{wins}(\mathsf{c})=\mathsf{true}} = 0.3$ and 
$\upr{\mathsf{wins}(\mathsf{c})=\mathsf{true}} = 0.3$.

Now if $\mathsf{move}(\mathsf{c},\mathsf{d})$ is discarded, the
well-founded model is the unique stable model where $\mathsf{b}$ is the only
winning position. 
But if $\mathsf{move}(\mathsf{c},\mathsf{d})$ is $\mathsf{true}$, then the
well-founded model assigns $\mathsf{true}$ to $\mathsf{wins}(\mathsf{c})$
and $\mathsf{false}$ to $\mathsf{wins}(\mathsf{d})$, leaving both
$\mathsf{wins}(\mathsf{a})$ and $\mathsf{wins}(\mathsf{b})$ as $\mathsf{undefined}$.
Hence the well-founded semantics yields
$\pr{\mathsf{wins}(\mathsf{c})=\mathsf{true}}=0.3$ and 
$\pr{\mathsf{wins}(\mathsf{c})=\mathsf{false}}=0.7$, while
$\pr{\mathsf{wins}(\mathsf{b})=\mathsf{true}}=0.7$ and 
$\pr{\mathsf{wins}(\mathsf{b})=\mathsf{undefined}}=0.3$.
$\Box$
%%%% CLINGO CODE:
%wins(X) :- move(X,Y), not wins(Y). 
%move(a,b).
%move(b,a).
%move(b,c).
%move(c,d).
%%%%%%%%%
\end{Example}

\begin{Example}\label{example:ProbabilisticBarberParadox}
Return to the Barber Paradox discussed in Example \ref{example:Barber},
now with a probabilistic twist:
\[
\begin{array}{c}
\mathsf{shaves}(X,Y) \colonminus 
	\mathsf{barber}(X), \mathsf{villager}(Y), \mathbf{not}\;\mathsf{shaves}(Y,Y). \\
\mathsf{villager}(a). \quad
\mathsf{barber}(b). \quad
0.5::\mathsf{villager}(b). 
\end{array}
\]

This program does not have a stable model when $\mathsf{villager}(b)$ is a fact.
Thus the {\sc plp} fails to have a credal semantics. 

However, the well-founded semantics is clear even when  $\mathsf{villager}(b)$ is $\mathsf{true}$:
in this case, $\mathsf{barber}(a)$, $\mathsf{shaves}(a,a)$ and $\mathsf{shaves}(a,b)$ are 
$\mathsf{false}$, while $\mathsf{shaves}(b,a)$ is $\mathsf{true}$, and $\mathsf{shaves}(b,b)$
is $\mathsf{undefined}$. And the well-founded semantics is also clear when 
$\mathsf{villager}(b)$ is discarded (that is, when $\mathsf{villager}(b)$ is $\mathsf{false}$): 
only  
$\mathsf{shaves}(b,a)$ is $\mathsf{true}$. Hence we obtain 
$\pr{\mathsf{shaves}(b,a)=\mathsf{true}}=1$, while
$\pr{\mathsf{shaves}(b,b)=\mathsf{false}}=0.5$ and
$\pr{\mathsf{shaves}(b,b)=\mathsf{undefined}}=0.5$.
\hfill $\Box$
\end{Example}

%%%% EXAMPLE IN PROBLOG (MALES AND FEMALES, AS IN PEARL):
%%% m1
%m1 :- f1, i1.
%m1 :- f2, i2.
%m1 :- i3.
%0.9::i1.
%0.8::i2.
%0.1::i3.
%%% m2
%m2 :- f1, i4.
%m2 :- f2, i5.
%m2 :- i6.
%0.7::i4.
%0.9::i5.
%0.2::i6.
%%% f1
%f1 :- m1, i7.
%f1 :- m2, i8.
%f1 :- i9.
%0.8::i7.
%0.8::i8.
%0.1::i9.
%%% f2
%f2 :- m1, i10.
%f2 :- m2, i11.
%f2 :- i12.
%0.8::i10.
%0.8::i11.
%0.2::i12.
%%%%
%evidence(f1,true).
%evidence(f2,true).
%evidence(m2,false).
%query(m1).
%%%%%%%%%%%%%%%%%%%%%%%

These examples should suffice to show that there are substantial differences
between the credal and the well-founded semantics. What to choose?

We start our analysis with the well-founded semantics.
At first it may seem that this semantics is very attractive because if attaches
a {\em unique} probability distribution to {\em every} well-formed {\sc plp}
(even in cases where the credal semantics is not defined). Besides, the 
well-founded semantics for {\sc plp}s is conceptually simple for
anyone who has mastered the well-founded semantics for normal logic 
programs. 

On the other hand, some of the weaknesses of the well-founded
semantics already  appear in non-probabilistic programs. 
Certainly the point here is not to emphasize non-probabilistic programs, 
but consider the difficulty of the well-founded semantics in ``reasoning by cases''.
For instance, consider the program \cite{Gelder91}:
\[
\mathsf{a} \colonminus \mathbf{not}\;\mathsf{b}. \qquad
\mathsf{b} \colonminus \mathbf{not}\;\mathsf{a}. \qquad
\mathsf{p} \colonminus \mathsf{a}. \qquad
\mathsf{p} \colonminus \mathsf{b}. 
\]
The well-founded semantics leaves every atom $\mathsf{undefined}$. However, 
it is apparent that $\mathsf{p}$ should be assigned $\mathsf{true}$, for we can 
find two ways to understand the relation between $\mathsf{a}$ and $\mathsf{b}$, 
and both ways take $\mathsf{p}$ to $\mathsf{true}$ (these two interpretations 
are exactly the stable models: one contains $\mathsf{a}$ and  $\neg \mathsf{b}$, 
the other contains $\neg \mathsf{a}$ and $\mathsf{b}$). 
%%%
The reader should note that this is similar to the situation in 
Example \ref{example:GraphColoring}: there  the well-founded semantics cannot
even fix the color of vertex $\mathsf{3}$, even though this vertex must clearly
be colored $\mathsf{red}$. 

The well-founded semantics of non-probabilistic normal logic programs has also 
drawn criticism in its reliance on three-valued
logic, and the status of the $\mathsf{undefined}$ truth value has received attention 
not only in philosophical inquiry \cite{Bergmann2008,Malinowski2007},
but in the practical development of databases \cite{Date2005,Rubinson2007}.
In short, it is difficult to determine whether $\mathsf{undefined}$ should be taken as simply 
an expression of subjective ignorance, or the indication that something really is 
neither $\mathsf{true}$ nor $\mathsf{false}$ \cite[Section 1.2.1.2]{Wallace93}.

In any case, we do not want to repeat the old and unresolved debate on three-valued logic
here; we want to focus on the even bigger problems that appear when three-valued logic
is mixed with probabilities.
The problem is that $\mathsf{undefined}$ values reflect a type of uncertainty,
and probability is supposed to deal with uncertainty; by putting those together
we may wish to invite collaboration but we may end up with plain confusion.
Consider for instance Example \ref{example:CyclicProgram}. What does it mean to
say that $\pr{\mathsf{headache}=\mathsf{undefined}}=0.165$? Supposedly
probability is here to tell us the odds of $\mathsf{true}$ and $\mathsf{false}$;
by learning the probability of $\mathsf{undefined}$, the next question should
be about the probability of $\mathsf{headache}$ to be $\mathsf{true}$ when
one is saying that it is $\mathsf{undefined}$. 
In fact, one might ask for the value of
$\pr{\mathsf{headache}=\mathsf{true}|\mathsf{headache}=\mathsf{undefined}}$,
not realizing that in the well-founded semantics this value is simply zero.
To emphasize the difficulty in interpretation, 
suppose we add to Example \ref{example:CyclicProgram} the simple rule
\[
\mathsf{c} \colonminus \mathsf{a}, \mathsf{b}\bigperiod
\]
and one asks for 
$\pr{\mathsf{c}=\mathsf{false}|\mathsf{cold}=\mathsf{undefined}}$.
Should this number really be $1$, as obtained through the well-founded semantics,
or should we expect this to be a question about $\pr{\mathsf{c}=\mathsf{false}}$,
given that nothing of substance is observed about $\mathsf{cold}$?

The probabilistic Barber Paradox discussed in Example \ref{example:ProbabilisticBarberParadox}
describes a situation where the well-founded semantics can answer questions for
some individuals, even as it fails to find definite answers for other questions.
This is rather attractive, but one must ask: What exactly is the meaning of
$\pr{\mathsf{shaves}(b,b)=\mathsf{undefined}}=0.5$? 
Note that, for the {\em logical program} described in Example \ref{example:Barber}, 
it makes sense to return an $\mathsf{undefined}$ value: we are at a logical corner.
However, for the probabilistic program it is less sensible to obtain a non-zero 
 probability that some particular fact is $\mathsf{undefined}$. 

A difficulty here is that $\mathsf{undefined}$ values appear due to a
variety of situations that should apparently be treated in distinct ways: 
(i) programs may be contradictory (as it happens in Example \ref{example:CyclicProgram});
(ii) programs may fail to have a clear meaning (as in the Barber Paradox);
or (iii) programs may simply have several possible meanings (for instance, various stable models
as in Example \ref{example:GraphColoring}). 
In  case (i), it is even surprising that one would try to assign probabilities
to  contradictory cases. In cases (ii) and (iii), probabilities
may be contemplated, but then there is a confusing mix of probabilities and 
$\mathsf{undefined}$ values. 
The interpretation of the various possible meanings of $\mathsf{undefined}$, 
already difficult in three-valued logic, is magnified by the challenges in interpreting
probabilities.
%Third value seems to arise from a variety of considerations (also there are ontological and epistemological reasons):
%1) When there is inconsistency.
%2) When there is no real answer, as in the Barber paradox.
%Consider here also the Kleene difficulty with definition of mathematical predicates when you have division by zero (some truth values are then undefined). 
%3) When there are several possible answers (as in stable models). 
%Lukasiewicz used "future contingencies" to justify truth values that are not known now (perhaps more in line with justification 3, as they can be true/false).
%Urquhart sees the third value as a set {0,1} of two potential classical values; that is, also justification 3.

%By looking at such examples, one can understand Riguzzi's refusal in dealing with
%$\mathsf{undefined}$ values: as he writes, ``the uncertainty should be
%handled by the choices [that is, by the probabilistic facts] rather than by the
%semantics of negation.''  

Now consider the credal semantics. There are two possible criticisms one may raise against it.
First,  a program may fail to have a credal semantics: consider the probabilistic Barber Paradox.
Second, the credal semantics relies on sets of probability measures (credal sets), not on unique measures.
We examine these two points in turn.

The fact that some programs may fail to have a credal semantics is an annoyance in
that programs must be checked for consistency. However, as we have noted already,
some programs can seem contradictory, and in those cases one could argue that it
is appropriate not to have semantics.  
So, one may be perfectly satisfied with failure in Example \ref{example:CyclicProgram},
for the total choice $\{\neg \mathsf{a},\mathsf{b}\}$ in essence leads to the following 
clearly unsatisfiable pair of rules:
\[
\mathsf{headache} \colonminus \mathsf{cold}. \qquad \mathsf{cold} \colonminus \mathbf{not}\;\mathsf{headache}.  
\]
 What seems to be needed here is a verifier that
checks consistency of {\sc plp}s; we look into this later in this paper. 

Now consider the fact that the credal semantics relies on credal sets.
Anyone expecting any inference to produce a single probability value may be puzzled,
but reliance on sets of probabilities does not seem to be a flaw when examined in detail. 
One argument in favor of sets of probabilities is that they are legitimate representations 
for incomplete, imprecise or indeterminate beliefs, that can be justified in a variety of ways
\cite{Augustin2014,Troffaes2014,Walley91}. But even if one is not willing to take
credal sets as a final representation of beliefs, the credal semantics is wholly 
reasonable from a least commitment perspective. That is, the main question
should always be: What are the best bounds on probabilities that one can safely assume, 
taking into account {\em only} the given rules, facts, and assessments? 
From this point of view, Examples \ref{example:ProbabilisticBasic}, \ref{example:GraphColoring},
and \ref{example:Wins} are entirely justified: the options given to the program are not 
decided by the given information, so one must leave them open. In particular
Example \ref{example:GraphColoring} seems to be an excellent argument for
the credal semantics: basically, the program generates all 3-colorings of a given graph;
why should we insist on singling out a distribution over colorings when no preference
over them is expressed?

All in all, we find that the credal semantics is conceptually stronger than the well-founded
semantics, even though the latter is uniquely defined for every {\sc plp}. We 
now examine the structural and computation properties of these two semantics;
a final comparison is left to Section \ref{section:Conclusion}.

\section{The structure of credal semantics}
\label{section:Structure}
 	
Given the generality of {\sc plp}s, one might think that credal sets 
generated by the credal semantics could have an arbitrarily complex structure.
Surprisingly, the structure of the credal semantics of a
{\sc plp} is a relatively simple object:

\begin{Theorem}
\label{theorem:Capacity}
Given a consistent {\sc plp},
its credal semantics is a set of probability measures that dominate an infinitely
monotone Choquet capacity.
\end{Theorem}

Before we present a proof of this theorem, let us pause and define a few terms.
An infinitely monotone Choquet capacity is a set function $\lpro$ from an algebra
$\mathcal{A}$ on a set $\Omega$ to the real interval $[0,1]$ such that 
\cite[Definition 4.2]{Augustin2014}:
$\lpr{\Omega}=1-\lpr{\emptyset}=1$ and, for any $A_1,\dots,A_n$ in the algebra,
$\lpr{\cup_i A_i} \geq \sum_{J \subseteq \{1,\dots,n\}} (-1)^{|J|+1} \lpr{\cap_{j \in J} A_j}$.
Infinitely monotone Choquet capacities appear in several formalisms; for instance,
they are the {\em belief functions} of Dempster-Shafer theory \cite{Shafer76}, 
 summaries 
of {\em random sets}~\cite{Molchanov2005}, and {\em inner measures} \cite{Fagin91}. 

Given an infinitely monotone  Choquet capacity
 $\lpro$, we can construct a set of measures that {\em dominate}
$\lpro$; this is the set $\{\pro : \forall A \in \mathcal{A}: \pr{A} \geq \lpr{A}\}$. 
We abuse language and say that a set consisting of all measures that dominate
an infinitely monotone Choquet capacity is an {\em infinitely monotone credal set}. 
If a credal set $\credalo$ is infinitely monotone, then the {\em lower probability} $\lpro$, defined as 
$\lpr{A} = \inf_{\pro \in \credalo} \pr{A}$, is exactly the generating infinitely monotone Choquet capacity.
We also have the {\em upper probability} $\upr{A}=\sup_{\pro \in \credalo} \pr{A} = 1-\lpr{A^c}$.

\begin{proof}[Proof of Theorem \ref{theorem:Capacity}]
Consider a set $\Theta$ containing as states the posssible total choices of the {\sc plp}.
 Over this space we have a product measure that is
completely specified by the probabilities attached to probabilistic facts.
 Now consider a multi-valued mapping $\Gamma$  between $\Theta$
and the space $\Omega$ of all possible  models of our probabilistic logic program. 
For each element $\theta \in \Theta$, define $\Gamma(\theta)$ to be
the set of  stable models associated with the total choice $\theta$ of 
the probabilistic facts. Now we use the fact that a probability space and a multi-valued
mapping induce an infinite monotone Choquet capacity over the range of the mapping
(that is, over $\Omega$) \cite{Molchanov2005}. 
\end{proof}

Infinitely monotone credal sets have several useful properties; for one
thing they are closed and convex. Convexity here means that if
$\pro_1$ and $\pro_2$ are in the credal set, then $\alpha \pro_1 + (1-\alpha) \pro_2$
is also in the credal set for $\alpha \in [0,1]$. 
Thus, as illustrated by Example \ref{example:ProbabilisticBasic}.

\begin{Corollary}
Given a consistent {\sc plp},
its credal semantics is a closed and convex set of probability measures.
\end{Corollary}

There are   several additional results concerning the representation of infinitely 
monotone capacities using their M\"obius transforms \cite{Augustin2014,Shafer76};
we refrain from mentioning every possible corollary we might produce by rehashing
those results. Instead, we focus on a few important results that can be used to
great effect in future applications.
First, as we have a finite Herbrand base, we can use the symbols in the proof
of Theorem \ref{theorem:Capacity} to write, for any set of models $\mathcal{M}$
\cite[Section 5.3.2]{Augustin2014}:
\begin{equation}
\label{equation:Bounds}
\textstyle
\lpr{\mathcal{M}} = \sum_{\theta \in \Theta : \Gamma(\theta) \subseteq \mathcal{M}} \pr{\theta}, 
\quad
\upr{\mathcal{M}} =  \sum_{\theta \in \Theta : \Gamma(\theta) \cap \mathcal{M} \neq \emptyset}  \pr{\theta}.
\end{equation}
Suppose we are interested in the
probability of a set $\mathbf{Q}$ of truth assignments for ground atoms in
the Herbrand base of the union of program $\mathbf{P}$ with all facts in $\mathbf{PF}$. 
A direct translation of Expression~(\ref{equation:Bounds})
leads to an algorithm that computes bounds on $\pr{\mathbf{Q}}$ as follows:
\begin{itemize}
\item[$\bullet$] {\bf Given} a {\sc plp} $\left<\mathbf{P},\mathbf{PF}\right>$
and $\mathbf{Q}$, initialize $a$ and $b$ with $0$. 
\item[$\bullet$] For each total choice $\theta$ of probabilistic facts, compute the set $S$ of 
all stable models of $\mathbf{P} \cup \mathbf{PF}^{\downarrow\theta}$, and: 
\begin{itemize}
\item if $\mathbf{Q}$ is $\mathsf{true}$ in every stable model in $S$, then 
$a \leftarrow a + \pr{\theta}$;  
\item if $\mathbf{Q}$ is $\mathsf{true}$ in some stable model of $S$, then 
$b \leftarrow b + \pr{\theta}$. 
\end{itemize}
\item[$\bullet$] {\bf Return} $[a, b]$ as the interval $[\lpr{\mathbf{Q}},\upr{\mathbf{Q}}]$.
\end{itemize}
Note that to find whether $\mathbf{Q}$ is $\mathsf{true}$ in every stable model of a program,
we must run {\em cautious inference}, and to find whether $\mathbf{Q}$ is $\mathsf{true}$
in some stable model of a program, we must run {\em brave inference}.
The complexity of these logical inferences have been studied in depth in the literature \cite{Eiter2007}.

For infinitely monotone credal sets we can find easy expressions for
  lower and upper {\em conditional} probabilities (that is, the infimum and
supremum of conditional probabilities). Indeed, if $A$ and $B$ are events, then
the lower  probability of $A$ given $B$ is
(where the superscript $c$ denotes complement)~\cite{Augustin2014}:
\begin{equation}
\label{equation:BayesRuleA}
\lpr{A|B} = \frac{\lpr{A \cap B}}{\lpr{A \cap B} + \upr{A^c \cap B}}
\end{equation}
when $\lpr{A \cap B}+\upr{A^c \cap B}>0$; we then have that
$\lpr{A|B}=1$ when $\lpr{A \cap B}+\upr{A^c \cap B}=0$ and $\upr{A \cap B}>0$;
finally, $\lpr{A|B}$ is undefined when $\upr{A \cap B} = \upr{A^c \cap B}=0$ (as this
condition is equivalent to $\upr{B}=0$). 
Similarly, the upper probability of $A$ given $B$ is:
\begin{equation}
\label{equation:BayesRuleB}
\upr{A|B} = \frac{\upr{A \cap B}}{\upr{A \cap B} + \lpr{A^c \cap B}}
\end{equation}
when $\upr{A \cap B}+\lpr{A^c \cap B}>0$; and we have that
$\upr{A|B} = 0$ when $\upr{A \cap B}+\lpr{A^c \cap B}=0$ and $\upr{A^c \cap B}> 0$;
finally, $\upr{A|B}$ is undefined when $\upr{A \cap B} = \upr{A^c \cap B}=0$.
%%%
We  also note that the computation of lower and upper {\em expected values}
with respect to infinitely monotone Choquet capacities 
admits relatively simple expressions~\cite{Wasserman92Bounds}.
For instance, the {\em lower expectation} $\lex{f} = \inf_{\pro\in\credalo}\ex{\pro}{f}$,
where $f$ is a function over the truth assignments, and $\ex{\pro}{f}$ is the expectation
of $f$ with respect to $\pro$, is
$\lex{f} = \sum_{\theta\in\Theta} \max_{\omega\in\Gamma(\theta)}f(\omega)$
\cite{Wasserman92Bounds}. And there are expressions even for lower and upper
{\em conditional expectations} that mirror Expression (\ref{equation:Bounds}). 

To translate these expresions into actual computations,
 suppose we have a  sets $(\mathbf{Q}$ and  $\mathbf{E}$ of truth assignments for ground atoms in
the Herbrand base of the union of program $\mathbf{P}$ with all facts in $\mathbf{PF}$.
To obtain bounds on $\pr{\mathbf{Q}|\mathbf{E}}$, we can combine the previous algorithm
with Expressions  (\ref{equation:BayesRuleA}) and (\ref{equation:BayesRuleB}), to obtain:
\begin{itemize}
\item[$\bullet$] {\bf Given} a {\sc plp} $\left<\mathbf{P},\mathbf{PF}\right>$
and $\mathbf{Q}$, initialize $a$, $b$, $c$, and $d$ with $0$. 
\item[$\bullet$] For each total choice $\theta$ of probabilistic facts, compute the set $S$ of 
all stable models of $\mathbf{P} \cup \mathbf{PF}^{\downarrow\theta}$, and:
\begin{itemize}
\item if $\mathbf{Q}\cup\mathbf{E}$ is $\mathsf{true}$ in every stable model in $S$, then 
$a \leftarrow a + \pr{\theta}$; 
\item if $\mathbf{Q}\cup\mathbf{E}$ is $\mathsf{true}$ in some stable model of $S$, then 
$b \leftarrow b + \pr{\theta}$; 
\item if $\mathbf{Q}$ if $\mathsf{false}$ and $\mathbf{E}$ is $\mathsf{true}$ in every 
stable model of $S$, then 
$c \leftarrow c + \pr{\theta}$; 
\item if $\mathbf{Q}$ if $\mathsf{false}$ and $\mathbf{E}$ is $\mathsf{true}$ in some 
stable model of $S$, then 
$d \leftarrow d + \pr{\theta}$. 
\end{itemize}
\item[$\bullet$] {\bf Return} the interval 
$[\lpr{\mathbf{Q}|\mathbf{E}},\upr{\mathbf{Q},\mathbf{E}}]$ as follows, in case
$b+d>0$ (otherwise, report failure and stop):
\begin{itemize}
\item $[0,0]$ if $b+c=0$ and $d>0$;
\item $[1,1]$ if $a+d=0$ and $b>0$;
\item $[a/(a+d),b/(b+c)]$ otherwise.
\end{itemize}
\end{itemize}

In fact the algorithm above has already been derived by  \citeA{Cali2009}, 
using clever optimization techniques 
(note that Cali et al.\ use a different strategy to handle the case where $\upr{\mathbf{E}}=0$). 
The advantage of our approach is that the algorithm is a transparent consequence
of known facts about capacities; other than that, Cali et al.\ have already presented
the algorithm so we do not need to dwell on it. Rather, we later return to this algorithm
with a focus on the complexity of computing of the lower probability $\lpr{\mathbf{Q}|\mathbf{E}}$. 
Algorithms that reproduce some properties of infinitely monotone Choquet capacities
are also presented by \citeA{Michels2015} in their work on constraint
logic programming. 

 \section{The complexity of inferences: acyclic and stratified probabilistic logic programs}
\label{section:ComplexityAcyclicStratified}

In this section we focus on the computation of inferences for acyclic and stratified {\sc plp}s; 
in these cases both the credal and the well-founded semantics agree. We focus on 
the following decision problem: 
\begin{description}
\item[Input:] A {\sc plp} $\left<\mathbf{P},\mathbf{PF}\right>$
whose probabilities are rational numbers,
a pair  $(\mathbf{Q}, \mathbf{E})$, called the {\em query}, 
where both $\mathbf{Q}$ and $\mathbf{E}$ are sets of truth assignments to
atoms in the Herbrand base of the union of program $\mathbf{P}$ and all facts in $\mathbf{PF}$, 
and a rational $\gamma \in [0,1]$. 
\item[Output:] Whether or not $\pr{\mathbf{Q}|\mathbf{E}}>\gamma$;
by convention, output is NO  (that is, input is rejected) if $\pr{\mathbf{E}}=0$. 
\end{description}
We refer to this complexity as the {\em inferential complexity} of {\sc plp}s. 
One may also be interested in the complexity of inferences
when the {\sc plp} is fixed, and the only input is the query $(\mathbf{Q},\mathbf{E})$.
This is the {\em query complexity} of the program; to define it, consider:
\begin{description}
  \item[Fixed:] A {\sc plp} $\left<\mathbf{P},\mathbf{PF}\right>$, whose 
   probabilities are rational numbers, that employs a vocabulary $\mathbf{R}$ of predicates.
  \item[Input:] A pair  $(\mathbf{Q}, \mathbf{E})$, called the {\em query}, 
where both $\mathbf{Q}$ and $\mathbf{E}$ are sets of truth assignments to
atoms of predicates in $\mathbf{R}$, 
and a rational $\gamma \in [0,1]$. 
  \item[Output:] Whether or not $\pr{\mathbf{Q}|\mathbf{E}}>\gamma$;
by convention, output is NO if $\pr{\mathbf{E}}=0$. 
\end{description}
Say that the query complexity of a class $\mathcal{P}$ of {\sc plp}s is in a complexity
class $\mathsf{C}$ if the complexity of this decision problem is in $\mathsf{C}$
for every {\sc plp} in $\mathcal{P}$. And say that the query complexity of
$\mathcal{P}$ is $\mathsf{C}$-hard if each decision problem in $\mathsf{C}$
can be reduced, with many-one reductions, to a decision problem for at least 
one {\sc plp} in $\mathcal{P}$. And say that the query complexity of
$\mathcal{P}$ is $\mathsf{C}$-complete
if $\mathcal{P}$ is both in $\mathsf{C}$ and $\mathsf{C}$-hard. 

In practice, one may face situations where a {\sc plp} may be small compared to 
the query, or where a single {\sc plp} is queried many times; then query complexity 
is the  concept of interest. 

The definition of  query complexity is clearly related to the concept of {\em data complexity} 
found in database theory \cite{Abiteboul95}; indeed we have used ``data complexity''  in previous
related work \cite{Cozman2015AAAI}.  Here we prefer to use ``query'' instead of ``data'' 
because usually data complexity fixes the rules and varies the number of facts; in this paper
 we keep both rules and facts fixed.  In fact, we have already mentioned the highly 
relevant work by  \citeA{Ceylan2016}, where they study the complexity of
various types of probabilistic logic programs; in that work they use {\em data complexity}
to refer to the complexity of computing probabilistic for fixed queries and fixed programs,
as the stock of facts and probabilistic facts varies. Note also that Ceylan et al.\ consider
a much more sophisticated language for queries that we do; for them, a query can be any union 
of Boolean conjunctive query as usually employed in databases \cite{Date2005}. 
The distinction between ``query'' and ``data'' thus seems significant in the context of
probabilistic logic programming. 

%In any case, query complexity and data 
%complexity of normal logic programs are intimately related. For instance, suppose that 
%(1) all probabilistic facts of a {\sc plp} contain no logical variables, and that
%(2) the data complexity of logical reasoning with the normal logic program obtained by 
%discarding probabilities is in some complexity class $\mathsf{C}$.
%Then the query complexity of the {\sc plp} is in
%$\mathsf{P}^\mathsf{C}$. Such an argument is particularly significant when $\mathsf{C}$ 
%is actually $\mathsf{P}$, as in this case the query complexity of the {\sc plp} is polynomial.  
%In fact, this sort of argument is the basis of the few existing tractability results for
%classes of probabilistic description logic programs \cite{Cali2009}. 

We must further comment on a few parallel results by \citeA{Ceylan2016}.
They analyze the complexity of {\sc plp}s under two semantics; one of them is in
line with Sato's distribution semantics, and another one is geared towards inconsistent
programs; neither is equivalent to the credal or the well-founded semantics.
Moreover, they focus on queries that are Boolean formulas, they do not allow for
conditioning evidence, and they use a somewhat different version of probabilistic
facts called contexts (that can be reproduced with our probabilistic facts). 
Despite these differences, in dealing with their first semantics they prove statements
that are related to results in Section \ref{subsection:Acyclic}.  More precisely: by 
translating the various languages and arguments appropriately, the points
made by our Theorems \ref{theorem:PropositionalAcyclic} and \ref{theorem:AcyclicBounded}
can be obtained from their results on full acyclic programs; also our Theorem \ref{theorem:QueryAcyclic} 
is comparable to their corresponding result, even though our ``query'' complexity is not their ``data''
complexity.\footnote{We note that our results on acyclic programs appeared \cite{Cozman2016PGM}
almost simultaneously to the publication by  \citeA{Ceylan2016}, and were produced independently.}
We decided to include our proof  of Theorem \ref{theorem:AcyclicBounded} in full
here because we need the techniques
in later proofs, and because we find that our techniques illuminate the matter adequately.

\subsection{Acyclic probabilistic logic programs}
\label{subsection:Acyclic}

We start with acyclic {\sc plp}s.
In this case the credal and the well-founded semantics define a single distribution, 
given by a Bayesian network whose structure is the program's grounded dependency 
graph, and whose parameters are
obtained from the program's Clark completion \cite{Poole93AI,Poole2008}.

\begin{Example}\label{example:Network}
Take a simplified version of the {\sc plp} in Example \ref{example:Alarm}, without predicates $\mathsf{calls}$
and $\mathsf{neighbor}$:
\[
\begin{array}{l}
0.7::\mathsf{burglary}. \qquad 0.2::\mathsf{earthquake}. \\
\mathsf{alarm} \colonminus \mathsf{burglary}, \mathsf{earthquake}, \mathsf{a1}.  \\
\mathsf{alarm} \colonminus \mathsf{burglary}, \mathbf{not}\ \mathsf{earthquake}, \mathsf{a2}. \\
\mathsf{alarm} \colonminus  \mathbf{not}\ \mathsf{burglary}, \mathsf{earthquake}, \mathsf{a3}.  \\
  0.9::\mathsf{a1}. \quad  0.8::\mathsf{a2}. \quad 0.1::\mathsf{a3}. 
\end{array}
\]
We can understand this {\sc plp} as the specification of the Bayesian network
in Figure \ref{figure:Alarm}. Note that the structure of the
network is just the grounded dependency graph, and the logical sentence comes
directly from the Clark completion. 
$\hfill \Box$
\end{Example}

\begin{figure}
\begin{center}
\begin{tikzpicture}[scale=1.3]
\node[draw,rectangle,rounded corners] (b) at (1,0.6) {$\mathsf{burglary}$};
\node[draw,rectangle,rounded corners] (e) at (1,-0.1) {$\mathsf{earthquake}$};
\node[draw,rectangle,rounded corners] (a) at (3.5,0.6) {$\mathsf{alarm}$};
\node[draw,rectangle,rounded corners] (a1) at (2.5,-0.1) {$\mathsf{a1}$};
\node[draw,rectangle,rounded corners] (a2) at (3.5,-0.1) {$\mathsf{a2}$};
\node[draw,rectangle,rounded corners] (a3) at (4.5,-0.1) {$\mathsf{a3}$};

\node[anchor=east] at (0,0.6) {\small $\pr{\mathsf{burglary}=\mathsf{true}}=0.7$};
\node[anchor=east] at (0,-0.1) {\small $\pr{\mathsf{earthquake}=\mathsf{true}}=0.2$};
\node[anchor=east] at (0,-0.7) {\small $\pr{\mathsf{a1}=\mathsf{true}}=0.9$};
\node[anchor=east] at (0,-1.1) {\small $\pr{\mathsf{a2}=\mathsf{true}}=0.8$};
\node[anchor=east] at (0,-1.5) {\small $\pr{\mathsf{a3}=\mathsf{true}}=0.1$};

\node at (3.5,-1.1) {\small 
$\begin{array}{ccl}
\mathsf{alarm} &  \Leftrightarrow & 
(\mathsf{burglary}\wedge\mathsf{earthquake}\wedge\mathsf{a}) \vee \\
& & (\mathsf{burglary}\wedge\neg\mathsf{earthquake}\wedge\mathsf{b}) \vee \\
& & (\neg\mathsf{burglary}\wedge\mathsf{earthquake}\wedge\mathsf{c})
\end{array}$};
 
\draw[->,>=latex] (b)--(a);
\draw[->,>=latex] (e)--(a);  
\draw[->,>=latex] (a1)--(a);
\draw[->,>=latex] (a2)--(a);
\draw[->,>=latex] (a3)--(a);
\end{tikzpicture}
\vspace*{-4ex}
\end{center}
\caption{Bayesian network extracted from the propositional
portion of Example \ref{example:Alarm}.}
\label{figure:Alarm}
\end{figure}
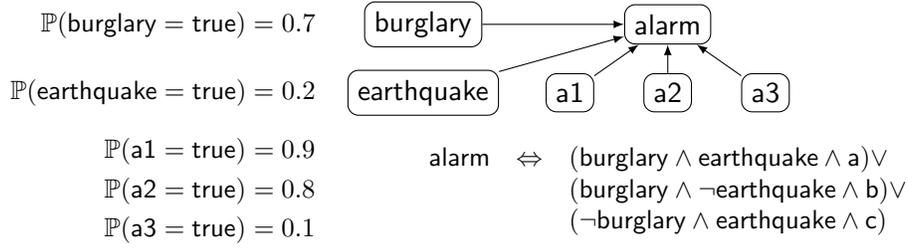

Conversely, any propositional Bayesian network can be specified
by an acyclic propositional {\sc plp}  \cite{Poole93AI,Poole2008}.
The argument is simple, and we show it by turning Example \ref{example:Network}
upside down:

\begin{Example}\label{example:Conversion}
Suppose we have the Bayesian network in Figure \ref{figure:AlarmShort}.
This Bayesian network is equivalent to the Bayesian network in
Figure \ref{figure:Alarm} (that is: the same distribution is defined
over $\mathsf{alarm}$, $\mathsf{burglary}$, $\mathsf{earthquake}$).
And the latter network is specified by an acyclic {\sc plp}. 
$\hfill \Box$
\end{Example}

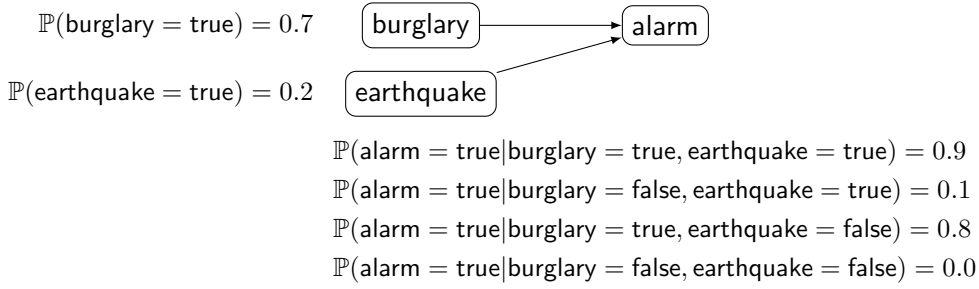
\begin{figure}
\begin{center}
\begin{tikzpicture}[scale=1.3]
\node[draw,rectangle,rounded corners] (b) at (1,0.6) {$\mathsf{burglary}$};
\node[draw,rectangle,rounded corners] (e) at (1,-0.1) {$\mathsf{earthquake}$};
\node[draw,rectangle,rounded corners] (a) at (3.5,0.6) {$\mathsf{alarm}$}; 

\node[anchor=east] at (0,0.6) {\small $\pr{\mathsf{burglary}=\mathsf{true}}=0.7$};
\node[anchor=east] at (0,-0.1) {\small $\pr{\mathsf{earthquake}=\mathsf{true}}=0.2$}; 
\node[anchor=west] at (0,-0.7) {\small $\pr{\mathsf{alarm}=\mathsf{true}|\mathsf{burglary}=\mathsf{true},\mathsf{earthquake}=\mathsf{true}}=0.9$};
\node[anchor=west] at (0,-1.5) {\small $\pr{\mathsf{alarm}=\mathsf{true}|\mathsf{burglary}=\mathsf{true}, \mathsf{earthquake}=\mathsf{false}}=0.8$};
\node[anchor=west] at (0,-1.1) {\small $\pr{\mathsf{alarm}=\mathsf{true}|\mathsf{burglary}=\mathsf{false},\mathsf{earthquake}=\mathsf{true}}=0.1$};
\node[anchor=west] at (0,-1.9) {\small $\pr{\mathsf{alarm}=\mathsf{true}|\mathsf{burglary}=\mathsf{false}, \mathsf{earthquake}=\mathsf{false}}=0.0$};

\draw[->,>=latex] (b)--(a);
\draw[->,>=latex] (e)--(a);   
\end{tikzpicture}
\vspace*{-4ex}
\end{center}
\caption{Bayesian network equivalent to the Bayesian network in Figure \ref{figure:Alarm}.}
\label{figure:AlarmShort}
\end{figure}

By combining these arguments, we see that inference
in  acyclic propositional {\sc plp}s has the complexity of inference in Bayesian
networks \cite{Darwiche2009,Roth96}:
\begin{Theorem}\label{theorem:PropositionalAcyclic}
The inferential complexity of inference in acyclic propositional {\sc plp}s  is $\mathsf{PP}$-complete.
\end{Theorem}

One might suspect that a bound on predicate arity would yield the same 
$\mathsf{PP}$-completeness, because the grounding of a {\sc plp} would then produce
only polynomially-many ground atoms.
Surprisingly, this is {\em not} the case here, as shown by the next theorem.

\begin{Theorem}
\label{theorem:AcyclicBounded}
The inferential complexity of inference in acyclic {\sc plp}s with bounded predicate
arity is $\mathsf{PP}^\mathsf{NP}$-complete.
\end{Theorem}
\begin{proof}
To prove membership, start with the ``unconditional'' decision $\pr{\mathbf{Q}}>\gamma$.
This decision problem is in $\mathsf{PP}^\mathsf{NP}$; this 
follows from the fact that logical reasoning with acyclic normal 
logic programs is $\Delta_2^P$-complete \cite[Table 5]{Eiter2007} (that is, 
$\mathsf{P}^\mathsf{NP}$-complete).  Consider a nondeterministic Turing
machine that goes through all probabilistic facts. For each probabilistic fact
$\alpha::A\bigperiod$, where $\alpha$ is a rational such that $\alpha=\mu/\nu$
for some (smallest) integers $\mu$ and $\nu$, the machine nondeterministically
decides whether to keep $A$ as a fact or discard it; then the machine creates
$\mu$ computation paths if $A$ is to be kept (all these computational paths reach
the same point), and $\nu-\mu$ computation paths if $A$ is to be discarded (again,
these computation paths reach the same point).
Note that these computation paths can be created with
polynomial effort even if $\mu$ and $\nu$ are specified in binary notation. 
Then, for the particular selection of probabilistic facts that are not discarded,
the machine processes the resulting  acyclic normal logic program:
logical reasoning can determine whether any set of truth 
assignments, for atoms in the Herbrand base, hold or not. 
The input is in the language if more than half the computation paths of this
machine are accepting paths. 
This decides whether $\pr{\mathbf{Q}}>1/2$.

Now consider membership of the decision $\pr{\mathbf{Q}|\mathbf{E}}>\gamma$;
we process this decision as follows.\footnote{We are indebted to Cassio Polpo de Campos 
for suggesting this technique; the probabilities attached to probabilistic facts are as proposed
by Park and described by  \citeA[Theorem 11.5]{Darwiche2009}.} 
Suppose the query consists of $\mathbf{Q}=\{Q_1,\dots,Q_n\}$ and $\mathbf{E}=\{E_1,\dots,E_m\}$,
where each $Q_i$ and each $E_j$ is a literal.  The simpler case is $\gamma \geq 1/2$, 
so assume it to begin. Then, as the query is processed, introduce 
$\mathsf{aux1} \colonminus Q_1,\dots,Q_n\bigperiod$,
$\mathsf{aux2} \colonminus E_1,\dots,E_m\bigperiod$,
$\mathsf{aux3} \colonminus \mathsf{aux1}, \mathsf{aux2},\mathsf{aux4}\bigperiod$,
$\mathsf{aux3} \colonminus \mathbf{not}\;\mathsf{aux2}, \mathsf{aux5}\bigperiod$,
where each literal $Q_i$ or $E_j$ is written as the corresponding subgoal, and
$(1/(2\gamma))::\mathsf{aux4}\bigperiod$, $0.5::\mathsf{aux5}\bigperiod$.
Thus
$\pr{\mathsf{aux3}=\mathsf{true}}>1/2 \Leftrightarrow 
(1/(2\gamma)) \pr{\mathbf{Q},\mathbf{E}} + (1/2) (1-\pr{\mathbf{E}}) > 1/2
\Leftrightarrow \pr{\mathbf{Q}|\mathbf{E}}>\gamma$ (that is, the decision
on $\pr{\mathsf{aux3}=1}>1/2$ yields the decision on $\pr{\mathbf{Q}|\mathbf{E}}>\gamma$).
Now if $\gamma<1/2$, then introduce
$\mathsf{aux3} \colonminus \mathbf{not}\;\mathsf{aux1}, \mathsf{aux2}, \mathsf{aux6}\bigperiod$,
remove $\mathsf{aux4}$,  and introduce 
   $(1-2\gamma)/(2-2\gamma)::\mathsf{aux6}\bigperiod$.
Thus
$\pr{\mathsf{aux3}=\mathsf{true}}>1/2 \Leftrightarrow 
\pr{\mathbf{Q},\mathbf{E}} + 
(1-2\gamma)/(2-2\gamma) (\pr{\mathbf{E}}-\pr{\mathbf{Q},\mathbf{E}})+
(1/2) (1-\pr{\mathbf{E}}) > 1/2
\Leftrightarrow \pr{\mathbf{Q}|\mathbf{E}}>\gamma$, as desired.
(This technique is used several times in latter proofs.)

Hardness is shown  by building a {\sc plp} that solves the problem $\#_1\mathsf{3CNF}(>)$
as defined in Section \ref{subsection:Complexity}; that is, one has a propositional
sentence $\phi$ in 3CNF with two sets of logical variables $\mathbf{X}$ and $\mathbf{Y}$,
and the goal is to decide whether the number of truth assignments for $\mathbf{X}$ that satisfy 
$\exists \mathbf{Y} : \phi(\mathbf{X},\mathbf{Y})$ is larger than a given integer $M$. 
We take $\phi$ to be a conjunction of clauses $c_1, \dots c_k$; each clause $c_j$ contains
an ordered triplet of propositional variables. 

For instance, we might have as input the integer $M=1$ and the formula
\begin{equation}
\label{equation:ProofExample}
\varphi(x_1,x_2,y_1) \equiv (\neg x_1 \vee x_2 \vee y_1) \wedge (x_1 \vee \neg x_2 \vee y_1) 
      \wedge (\neg y_1 \vee \neg y_1 \vee \neg y_1).
%\vspace{-0.8em}
\end{equation}
In this case the input is accepted (the number of satisfying assignments is $2$).
Note that the last clause is equivalent to $\neg y_1$; we pad the clause so as to have three
literals in it. 

For each propositional variable $y_i$, we introduce a corresponding logical variable $Y_i$.
The ordered tuple of propositional variables in clause $c_j$ corresponds to a tuple
of propositional variables that is denoted by  $\mathbf{Y}_j$; these are the propositional
variables in $c_j$ that belong to $\mathbf{Y}$. In Expression 
(\ref{equation:ProofExample}), $\mathbf{Y}_1=\mathbf{Y}_2=\mathbf{Y}_3=[Y_1]$.

We use a few constants and predicates.
Two constants, $\mathsf{0}$ and $\mathsf{1}$, stand for 
$\mathsf{false}$ and $\mathsf{true}$ respectively. 
Also, we use $0$-arity predicates $\mathsf{x_i}$, each one standing
for a propositional variable $x_i$ in $\mathbf{X}$.  
And we use predicates $\mathsf{c_1},\ldots,\mathsf{c_k}$, 
each one standing for a clause $c_j$. The arity of each $\mathsf{c_j}$ is the 
length of $\mathbf{Y}_j$, denoted by $d_j$. 

For each  $\mathsf{c_j}$, go over the $2^{d_j}$ possible   assignments of $\mathbf{Y}_j$.
That is, if $d_j=1$, then go over $\mathsf{c_j}(\mathsf{0})$ and $\mathsf{c_j}(\mathsf{1})$;
if $d_j=2$, then go over $\mathsf{c_j}(\mathsf{0},\mathsf{0})$, $\mathsf{c_j}(\mathsf{0},\mathsf{1})$,
$\mathsf{c_j}(\mathsf{1},\mathsf{0})$ and $\mathsf{c_j}(\mathsf{1},\mathsf{1})$. And if $d_j=3$, 
go over the $8$ assignments. Note that if $d_j=0$, there is only one ``empty'' assignment to visit.
Thus there are at most $8k$ assignments to consider. 

Suppose then that we have predicate $\mathsf{c_j}$, and we take the assignment 
$\mathbf{y}$ (which may be empty).  
If $c_j$ is $\mathsf{true}$ for $\mathbf{y}$, regardless of the possible assignments
for propositional variables $x_i$, then just introduce the fact
\[
\mathsf{c_j}(\mathbf{y})\bigperiod
\]
If instead $c_j$ is $\mathsf{false}$ for $\mathbf{y}$, regardless of the possible
assignments for propositional variables $x_i$, then just move to another assignment
(that is, there are no propositional variables $x_i$ in $c_j$, and the clause is $\mathsf{false}$
for $\mathbf{y}$; by leaving $\mathsf{c_j}(\mathbf{y})$, we guarantee that it is forced
to be $\mathsf{false}$ by the semantics). 
Otherwise, there are propositional variables in $\mathbf{X}$ that affect the truth
value of $c_j$ when $\mathbf{y}$ is fixed; there may be one, two or three such propositional
variables. Take the first one of them, denoted by $x_{j1}$, and introduce the rule
\[
\mathsf{c_j}(\mathbf{y}) \colonminus 
\left\{ \begin{array}{ll}
\mathsf{x_{j1}}\bigperiod & \mbox{ if the literal for } x_{j1} \mbox{ does not contain negation; or } \\
\mathbf{not}\;\mathsf{x_{j1}}\bigperiod & \mbox{ if the literal for } x_{j1} \mbox{ contains negation.} 
\end{array} \right.
\]
If there is  a second propositional variable $x_{j2}$ that affects the truth value of $c_j$ when
$\mathbf{y}$ is fixed, add a similar rule 
$\mathsf{c_j}(\mathbf{y}) \colonminus [\mathbf{not}]\;\mathsf{x_{j2}}$. 
And similarly if there is a third propositional variable $x_{j3}$ that affects the truth value of $c_j$. 
Note that these rules create a disjunction for $\mathsf{c_j}$, in effect encoding the clause $c_j$
for fixed $\mathbf{y}$. 

Finally, introduce the rule
\[
\mathsf{cnf} \colonminus \mathsf{c_1}(\mathbf{Y}_1), \mathsf{c_2}(\mathbf{Y}_2), 
     \dots, \mathsf{c_k}(\mathbf{Y}_k)\bigperiod
\]
and probabilistic facts (one per predicate $\mathsf{x_i}$)
\[
0.5::\mathsf{x_i}\bigperiod
\]

The Clark completion of the {\sc plp} just constructed encodes the 
$\#_1\mathsf{3CNF}(>)$ problem of interest, 
 thus proving $\mathsf{PP}^\mathsf{NP}$-hardness: to determine
whether $\exists \mathbf{Y}: \phi(\mathbf{X},\mathbf{Y})$ has more
than $M$ satisfying assignments, decide whether $\pr{\mathsf{cnf}=\mathsf{true}}>M/2^{n}$,
where $n$ is the number of propositional variables in $\mathbf{X}$. 

For instance, given the formula in Expression~(\ref{equation:ProofExample}),
generate the following  {\sc plp}: 
\[
\begin{array}{l}
\mathsf{c_1}(\mathsf{0}) \colonminus  \mathbf{not} \;\mathsf{x_1}\,\bigperiod 
\qquad
\mathsf{c_1}(\mathsf{0}) \colonminus  \mathsf{x_2}\,\bigperiod 
\qquad
\mathsf{c_1}(\mathsf{1})\bigperiod \\
%%%
\mathsf{c_2}(\mathsf{0}) \colonminus \mathsf{x_1}\,\bigperiod 
\qquad
\mathsf{c_2}(\mathsf{0}) \colonminus \mathbf{not} \;\mathsf{x_2}\,\bigperiod 
\qquad
\mathsf{c_2}(\mathsf{1})\bigperiod \\
%%%
\mathsf{c_3}(\mathsf{0})\bigperiod \\
\mathsf{cnf} \colonminus \mathsf{c_1}(Y_1), \mathsf{c_2}(Y_1), \mathsf{c_3}(Y_1)\bigperiod \\
0.5::\mathsf{x_1}\bigperiod \qquad 0.5::\mathsf{x_2}\bigperiod
\end{array}
\]
By determining whether $\pr{\mathsf{cnf}=\mathsf{true}}>M/2^2$, we decide whether
the number of truth assignments for $x_1$ and $x_2$ such that $\forall y_1 \varphi(x_1,x_2,y_1)$ 
holds is larger than $M$.  
%%%% PROBLOG PROGRAM:
%c1(0) :- not x1.
%c1(0) :- x2.
%c1(1).
%c2(0) :- x1.
%c2(0) :- not x2.
%c2(1).
%c3(0).
%cnf :- c1(Y), c2(Y), c3(Y).
%0.5::x1.
%0.5::x2.
%query(cnf).
%%%%%
\end{proof}

Intuitively, this results shows that,
 to produce an inference for a {\sc plp} with bounded predicate arity, one must
go through the truth assignments for polynomially many groundings, guessing one 
at a time (thus a counting nondeterministic Turing machine), and, for {\em each} 
assignment, it is then necessary to use an $\mathsf{NP}$-oracle to construct the 
probability values.
%%%%
%This result applies to $\mathsf{PLP_b}(k)$ as well
%(to show this: if a conjunction uses more than $k$ atoms, just introduce fresh 
%predicates and break the conjunction in ``smaller'' conjunctions).
%%%
Theorem \ref{theorem:AcyclicBounded}
suggests that acyclic {\sc plp}s capture a larger set of probabilistic languages
than many probabilistic relational models that stay within $\mathsf{PP}$
\cite{CozmanSUM2015plates}.

The next step is to remove the bound on arity. We   obtain:
\begin{Theorem}\label{theorem:AcyclicEXP}
The inferential complexity of inference in acyclic {\sc plp}s
is $\mathsf{PEXP}$-complete.
\end{Theorem}
\begin{proof}
Membership follows from grounding the {\sc plp}.\footnote{A short proof 
of membership is obtained by applying the same concise argument used 
in the proof of Theorem \ref{theorem:Stratified}(c); here we present a longer 
but possibly more intuitive argument based on inference on Bayesian networks.}
If the {\sc plp} has $n$ constants, then a relation of arity $k$ produces
$n^k$ groundings. Each one of these exponentially many groudings
corresponds to a node of a (necessarily acyclic) Bayesian network.
To write down the conditional probabilities associated
with each node of the grounded Bayesian network, take the
Clark completion of the program, and ground the expresions. 
For each non-root node we have a first-order formula that can be
written as a possibly exponentially-long quantifier-free formula.
Now to determine whether $\pr{\mathbf{Q}}>1/2$, we can use
a probabilistic Turing machine that runs inference for the exponentially
large (grounded) Bayesian network (or, rather, there is an exponential-time
Turing machine that guesses a truth assignment for all grounded probabilistic
facs, and for each such truth assignment, computes the truth assignment
for any other atom by going through the possibly exponentially large
non-root node completions). 

To prove hardness, we encode an exponential-time nondeterministic
Turing machine $\mathbb{M}$ using logical formulas that are directly produced
by the Clark completion of an acyclic normal logic program $\mathbf{P}$.  Assume 
that $\mathbb{M}$ can solve some $\mathsf{PEXP}$-complete problem; that is, 
for a $\mathsf{PEXP}$-complete language $\mathcal{L}$, $\ell \in \mathcal{L}$ iff 
$\mathbb{M}$ halts within time $2^n$ with more than half of paths accepting $\ell$, 
where $n$ is some polynomial on the length of $\ell$.  We also add probabilistic facts
$\mathbf{PF}$ to $\mathbf{P}$, 
so that an inference in the resulting {\sc plp} decides whether the number of 
acceptings paths of $\mathbb{M}$ is larger than half of the total number of 
computation paths (thus deciding the same language as $\mathbb{M}$ does). 
So, consider the encoding of Turing machines that is described 
by  \citeA[Theorem 3.2.4]{Gradel2007}, summarized as follows. Suppose 
$\mathbb{M}$ has states $q$, with an initial state $q_0$, an accepting state 
$q_a$, and a rejecting state $q_r$; suppose also that $\mathbb{M}$ uses an 
alphabet with  symbols $\sigma$  (in the alphabet there is a blank symbol 
$\scriptstyle \sqcup$); finally suppose that $\mathbb{M}$ has a transition 
function $\delta$ that takes a pair $(q,\sigma)$, understood as state $q$
and symbol $\sigma$ read by the machine head, and returns one of a number
of triplets $(q', \sigma', m)$, where $q'$ is the next state, $\sigma'$ is the
symbol to be written at the tape, and $m$ is either $-1$ (head goes to the left),
$0$ (head stays at the same position), and $1$ (head goes to the right).
Assume that the alphabet is enlarged so that every pair $(q,\sigma)$ is also
a possible symbol.  The input to the machine is a sequence of symbols
 $(q_0,\sigma_0^1),\sigma_0^2,\dots,\sigma_0^m$, and a configuration
of the tape is $\sigma^1,\sigma^2,\dots,(q,\sigma),\dots,\sigma^{2n}$
(note that the ``useful'' portion of the tape runs from position $1$ to position
$2^n$). 

The encoding of $\mathbb{M}$ is obtained by introducing a number of 
predicates and a number of first-order sentences $\phi_i$; when all these
sentences hold, then any interpretation for the predicates is an accepting
computation. We omit the logical expressions of this encoding as they can
be taken from Gr\"adel's presentation. In any case, if we decide whether 
the number of interpretations for the predicates in these sentences is
larger than half of the number of possible interpretations, we obtain the
desired decision. We enforce each sentence $\phi_i$ by introducing 
a predicate $\mathsf{aux_i}$ and a rule $\mathsf{aux_i} \colonminus \phi_i$
  (where we write $\phi_i$ in the rule with the understanding
that $\phi_i$ is obtained as the Clark completion of a set of auxiliary predicates
and rules; recall that conjunction, disjunction and negation are available, as well
as existential quantifiers; universal quantifiers are produced by negating
existential ones); then the sentence $\phi_i$ holds when $\{\mathsf{aux_i}=\mathsf{true}\}$
holds. We simply collect all these truth assignments in the set $\mathbf{E}$.
Now, we must have one of the sentences in $\mathbb{M}$'s encoding as a 
``detector'' for  the accepting state; that is, $\exists X: \mathsf{state}_{q_a}(X)$, 
where $X$ indexes the computation
steps, and $\mathsf{state}_{q_a}(X)$ is a predicate that indicates that at computation
step $X$ the state is $\mathsf{state}_{q_a}$. Denote by $\mathsf{aux_a}$ the 
auxiliary predicate associated with the latter sentence. At this point we can
reproduce the behavior of $\mathbb{M}$ if we focus on interpretations that
satisfy $\mathbf{E}$. The next step is to encode the input.
Now, the input symbols can be inserted by appropriate facts (these facts refer
to predicates introduced in the encoding). 
And the final step is to count the accepting computations. First we must assume
that, once $\mathbb{M}$ reaches $q_a$ or $q_r$, it stays with the same configuration
(it just keeps repeating the state and the tape), so that the number of accepting
paths is the same number of interpretations that satisfy $\{\mathsf{aux_a}=\mathsf{true}\}$;
this assumption is harmless as $\mathbb{M}$ can always be modified to do it.
Then we add, for each predicate $\mathsf{r}$ that is introduced in the construction, except
the ones in $\mathbf{E}$, the probabilistic fact $0.5::\mathsf{r}(X_1,\dots,X_k)$, where
$k$ is the arity of $\mathsf{r}$. Given all of this, the decision
$\pr{\mathsf{aux_a}=\mathsf{true}|\mathbf{E}}>1/2$ determines whether the number of
``accepting'' interpretations for $\mathbb{M}$ is larger than half the number of intepretations
for $\mathbb{M}$. Thus hardness obtains.
\end{proof}

Consider   query complexity. The 
following result is handy:
\begin{Theorem}
\label{theorem:SingleRule}
Query complexity is $\mathsf{PP}$-hard for the following {\sc plp}:
\[
0.5::\mathsf{t}(X)\,\bigperiod \quad
0.5::\mathsf{pos}(X,Y)\,\bigperiod \quad
0.5::\mathsf{neg}(X,Y)\,\bigperiod
\]
\[
\mathsf{c}(Y) \colonminus \mathsf{pos}(X,Y), \mathsf{t}(X)\,\bigperiod \quad
\mathsf{c}(Y) \colonminus \mathsf{neg}(X,Y), \mathbf{not}\;\mathsf{t}(X)\,\bigperiod
\]
\end{Theorem}
\begin{proof}
  Consider a
  CNF formula $\varphi(x_1,\ldots,x_n)$ with clauses
  $c_1,\ldots,c_m$ and propositional variables $x_1,\dots,x_n$. 
  Let $P_j$ (resp., $N_j$) be a vector denoting the indices of the
  positive (negative) literals $x_i$ ($\neg x_i$) in clause $j$. We can
  encode the truth-value of a clause $c_j$ as
  $\mathsf{c}(j)$, the truth-value of  $x_i$ as $\mathsf{t}(i)$,
  and the occurrence of a  positive (negative) literal $x_i \in P_j$ ($x_i \in N_j$) as
  $\mathsf{pos}(i,j)$ ($\mathsf{neg}(i,j)$).
  So assemble a query $\mathbf{Q}$
  containing assignments to $\{\mathsf{c}(j)=\mathsf{true}\}$ for
  $j=1,\ldots,m$, $\{\mathsf{pos}(i,j)=\mathsf{true}\}$ for $i \in P_j, j =
  1,\ldots,m$, $\{\mathsf{neg}(i,j)=\mathsf{true}\}$ for $i \in N_j,
  j=1\ldots,m$. Now if a grounding of $\mathsf{pos}$ or $\mathsf{neg}$ is
 not already assigned $\mathsf{true}$, then assign it to $\mathsf{false}$
 and add this assignment to $\mathbf{Q}$.  The Clark completion
   defines $\mathsf{c}(j) \Leftrightarrow \bigvee_{i \in
    P_j} \mathsf{t}(i) \vee \bigvee_{i \in N_j} \neg \mathsf{t}(i)$ for  every
    $c_j$. And the number of assignments to $x_1,\dots,x_n$
  that satisfy $\varphi$ is larger than $M$ iff $\pr{\mathbf{Q}}>M/2^{2s^2+s}$
  where $s=\max(m,n)$; 
  hence the desired hardness obtains. 
\end{proof}

Consequently:
\begin{Theorem}\label{theorem:QueryAcyclic}
The query complexity of inference for acyclic {\sc plp}s is $\mathsf{PP}$-complete.
\end{Theorem}
\begin{proof}
Hardness follows from Theorem~\ref{theorem:SingleRule}. Membership is 
obtained using the same reasoning in the proof of Theorem \ref{theorem:AcyclicBounded},
only noting that, once the probabilistic facts are selected, logical reasoning
with the resulting acyclic normal logic program can be done with polynomial
effort \cite[Theorem 5.1]{Dantsin2001}; thus $\pr{\mathbf{Q}}>\gamma$ can be 
decided within $\mathsf{PP}$. 
\end{proof}

There are   subclasses of acyclic {\sc plp}s that characterize well-known
tractable Bayesian networks. An obvious one is the class of propositional acyclic programs
whose grounded dependency graph has  bounded treewidth, as  Bayesian networks subject 
to such a bound are tractable \cite{Koller2009}.  
%%%
As another interesting example, consider the two-level networks that are processed
by the {\em Quick-Score} algorithm \cite{Heckerman90or}; that is, two-level networks
where the top level consists of marginally independent ``diseases'' and the bottom
level consists of ``findings''  that are conditionally independent given the diseases, and 
that are determined by {\em noisy-or} gates. Such a network can be easily encoded
using a propositional acyclic {\sc plp}; these {\sc plp}s inherit the fact that inference
is polynomial when $\mathbf{Q}$ contains only negated atoms (that is, only $\mathsf{false}$).
% (actually, positive evidence is handled by Quick-Score, but the cost is exponential). 
Alas, this tractability result is quite fragile, as ``positive'' evidence breaks
polynomial behavior as long as $\mathsf{P} \neq \mathsf{NP}$~\cite{Shimony2003}.
%%%
Yet another tractable class consists of acyclic definite propositional {\sc plp}s 
   such that each atom is the head of at most one rule:
inference in this class is polynomial when $\mathbf{Q}$ contains only
$\mathsf{true}$. This is obtained by noting that the Clark completion of these 
programs produces Bayesian networks
that are specified using only conjunction, and a polynomial algorithm obtains
from results by  \citeA{Cozman2015AAAI}. This is also a fragile result:
\begin{Proposition} \label{hard-acyclic}
Inference for the class of acyclic  propositional {\sc plp}s such that 
each atom is the head of at most one rule is $\mathsf{PP}$-complete even if
(a)  $\mathbf{Q}$ contains only $\mathsf{true}$ but the program contains $\mathbf{not}$;
(b) the program is definite but $\mathbf{Q}$ contains $\mathsf{false}$.
%(c) the program is definite and evidence is positive.
\end{Proposition}
\begin{proof}
Membership follows, for both (a) and (b), from Theorem \ref{theorem:PropositionalAcyclic}.
So, consider hardness. Any {\sc plp} in Case (b) produces, as its Clark completion,
a Bayesian network that is specified using conjunctions; inference for this sort of
Bayesian network is $\mathsf{PP}$-complete when evidence can be 
``negative''~\cite{CozmanAIJreport2016}. Hardness for Case (a) then obtains easily, because
one can use negation to turn ``positive'' evidence into ``negative'' evidence.
%Case (c) follows from  Theorem \ref{theorem:SingleRule} (note that Case (c) includes 
%non-propositional programs).
\end{proof}

\subsection{Stratified probabilistic logic programs}
\label{subsection:Stratified}

%As we have noted already, a stratified normal logic program is one where the grounded dependency graph
%has no cycles containing a negative edge \cite{Apt88}.\footnote{Often
%such a  normal logic program is referred to as a {\em locally stratified normal logic programs}.}
%A stratified program can be divided into a sequence of disjoint sub-programs (each
%called a {\em stratum}), such that a predicate that appears positively (that is, without
%a {\bf not}) in a stratum is defined in this stratum
%and preceding ones, while a predicate that appears
%negatively (with a $\mathbf{not}$)  is defined in preceding strata. 
A stratified normal logic program has the useful property that its universally adopted semantics
produces a single interpretation (and is equal to its stable and well-founded semantics).
Because every total choice of a stratified {\sc plp} produces a stratified normal logic program,
the credal/well-founded semantics of a stratified {\sc plp} is a unique distribution. 

One might fear that in moving from acyclic to stratified programs we must pay a
large penalty. This is {\em not} the case: the complexity classes remain
the same as in Section \ref{subsection:Acyclic}:

\begin{Theorem}
\label{theorem:Stratified} 
For locally stratified {\sc plp}s, inferential complexity is $\mathsf{PEXP}$-complete;
it is $\mathsf{PP}^\mathsf{NP}$-complete for {\sc plp}s with bounded predicate arity;
it is $\mathsf{PP}$-complete for propositional {\sc plp}s.
For locally stratified {\sc plp}s, query complexity is $\mathsf{PP}$-complete.
\end{Theorem}
\begin{proof}
For propositional stratified {\sc plp}s, hardness comes from the fact that a Bayesian 
network on binary random variables can be encoded by a stratified program (indeed, 
by an acyclic program), and inference with such networks is $\mathsf{PP}$-complete
\cite{Darwiche2009,Roth96}. 
Membership is obtained using the same reasoning in the proof of 
Theorem \ref{theorem:AcyclicBounded}, only noting that, once the probabilistic facts 
are selected, logical reasoning with the resulting stratified normal logic program can 
be done with polynomial effort \cite[Table~2]{Eiter2007}.

For stratified programs with bounded predicate arity,
hardness  follows from Theorem~\ref{theorem:AcyclicBounded}. 
Membership is 
obtained using the same reasoning in the proof of Theorem \ref{theorem:AcyclicBounded};
in fact that proof of membership applies directly to stratified programs with bounded arity.

For general stratified {\sc plp}s,
hardness is argued as in the proof of  Theorem~\ref{theorem:AcyclicEXP}.
Membership follows from the fact that we can 
ground the {\sc plp} into an exponentially large propositional {\sc plp}.
Once the (exponentially-many) probabilistic facts are selected, the Turing
machine is left with a stratified propositional normal logic program, and
logical inference is polynomial in the size of this program (that is, logical
inference requires exponential effort). 

Finally, hardness of query complexity follows from Theorem \ref{theorem:SingleRule}. 
 Membership is 
obtained using the same reasoning in the proof of Theorem \ref{theorem:AcyclicBounded},
only noting that, once the probabilistic facts are selected, logical reasoning
with the resulting stratified normal logic program can be done with 
polynomial effort as guaranteed by the analysis of data complexity of
stratified normal logic programs~\cite{Dantsin2001}.
\end{proof}

We noted, at the end of Section \ref{subsection:Acyclic},
 that some sub-classes of acyclic programs display polynomial 
behavior. We now show an analogue result for a sub-class of definite
(and therefore stratified, but possibly cyclic) programs with unary
and binary predicates:

\begin{Proposition}\label{proposition:DLLite}
Inferential complexity is polynomial for queries containing only $\mathsf{true}$, for {\sc plp}s where: 
(a) every predicate is unary or binary, and facts can be asserted about them;
(b) probabilistic facts can be of the form $\alpha::\mathsf{a}(X)\bigperiod$, 
$\alpha::\mathsf{a}(a)$, $\alpha::\mathsf{r}(X,Y)$ (that is, each unary predicate
can be associated with ground or non-ground probabilistic facts, while each
binary predicate can be associated to a particular non-ground probabilistic fact); 
(c) no binary predicate is the head of a rule that has a body;
(d) each atom is the head of at most one rule that has a body, and only the three following
rule forms are allowed:
%%% IMPORTANT
%% QUESTION: CAN a APPEAR AS HEAD OF A RULE AND AS A PROBABILISTIC FACT??????
%%%%%%%
%%% PERHAPS IT WOULD BE EASIER TO REQUIRE THAT THE PROGRAM CONTAINS
%%% ONLY UNARY AND BINARY PREDICATES, AND THEN RESTRICT THE POSSIBLE RULES!!!!
%%%%
%%%%%%
\[
\mathsf{a}(X) \colonminus
      \mathsf{a}_1(X), \dots, \mathsf{a}_k(X)\,\bigperiod  \qquad
\mathsf{a}(X) \colonminus \mathsf{r}(X,Y)\,\bigperiod \qquad
      \mathsf{a}(X) \colonminus \mathsf{r}(Y,X)\,\bigperiod.
\] 
%\[
%\left\{
%    \begin{array}{c}
%      \alpha::\mathsf{a}(X)\,\bigperiod \quad \beta::\mathsf{a}(a)\,\bigperiod \quad \gamma::\mathsf{r}(X,Y)\,\bigperiod \quad
%      \mathsf{a}(a)\,\bigperiod \quad \mathsf{r}(a,b)\,\bigperiod \\ 
%      \mathsf{a}(X) \colonminus
%      \mathsf{a}_1(X), \dots, \mathsf{a}_k(X)\,\bigperiod \quad\quad
%      \mathsf{a}(X) \colonminus \mathsf{r}(X,Y)\,\bigperiod \quad \quad
%      \mathsf{a}(X) \colonminus \mathsf{r}(Y,X)\,\bigperiod
%    \end{array} \right.
%\]
\end{Proposition}
\begin{proof}
  We show that the inference can be reduced to a tractable weighted
  model counting problem.  First, ground the program in polynomial time
  (because each rule has at most two logical variables). Since the resulting program
  is definite, only atoms that are ancestors of the queries in the
  grounded dependency graph are relevant for determining the
  truth-value of the query in any logic program induced by a total
  choice (this follows as resolution is complete for propositional
  definite programs). Thus, discard all atoms that  are not ancestors of
  a query atom. For the query to be true, the remaining atoms that are
  not probabilistic facts are forced to be true by the semantics. So
  collect all rules of the sort
  $\mathsf{a}(a) \colonminus \mathsf{r}(a,b)\,\bigperiod$,
  $\mathsf{a}(a) \colonminus \mathsf{r}(b,a)\,\bigperiod$, plus all facts and all
  probabilistic facts. This is an acyclic program, so that its Clark
  completion gives the stable model semantics. This completion is a
  formula containing a conjunction of subformulas
  $\mathsf{a}(a) \Leftrightarrow \bigvee_b \mathsf{r}(a,b)$,
  $\mathsf{a}(a) \Leftrightarrow \bigvee_a \mathsf{r}(a,b)$, and unit
  (weighted) clauses corresponding to (probabilistic) facts. The query
  is satisfied only on models where the lefthand side of the definitions
  are true, which is equivalent to reducing the subformulas to their
  righthand side. The resulting weighted model counting problem has been
  shown to be polynomial-time solvable \cite{MauaSUM2015dlite}.
\end{proof}

%% The class of DL-Lite {\sc plp}s has polynomial query complexity even when
%% negative evidence is allowed on unary atoms. 

\section{The complexity of inferences: credal semantics}
\label{section:ComplexityCredal}

Now consider {\sc plp}s that may be non-stratified.
We have to adapt the definitions of inferential and query complexity to account
for the fact that we now have lower and upper probabilities. First we
focus on lower probabilities; the {\em lower-probability} version of inferential complexity 
for  a class of {\sc plp}s is the complexity of the following decision problem:

\begin{description}
\item[Input:] A {\sc plp} $\left<\mathbf{P},\mathbf{PF}\right>$
whose probabilities are rational numbers,
a pair  $(\mathbf{Q}, \mathbf{E})$, called the {\em query}, 
where both $\mathbf{Q}$ and $\mathbf{E}$ are sets of truth assignments to
atoms in the Herbrand base of the union of program $\mathbf{P}$ and all facts
in $\mathbf{PF}$, 
and a rational $\gamma \in [0,1]$. 
\item[Output:] Whether or not $\lpr{\mathbf{Q}|\mathbf{E}}>\gamma$;
by convention, output is NO  (that is, input is rejected) if $\upr{\mathbf{E}}=0$. 
\end{description}
The {\em lower-probability} version of query complexity is, accordingly: 
\begin{description}
  \item[Fixed:] A {\sc plp} $\left<\mathbf{P},\mathbf{PF}\right>$, whose 
   probabilities are rational numbers,  that employs a vocabulary $\mathbf{R}$ of
  predicates.
  \item[Input:] A pair  $(\mathbf{Q}, \mathbf{E})$, called the {\em query}, 
where both $\mathbf{Q}$ and $\mathbf{E}$ are sets of truth assignments to
atoms of predicates in $\mathbf{R}$, 
and a rational $\gamma \in [0,1]$. 
  \item[Output:] Whether or not $\lpr{\mathbf{Q}|\mathbf{E}}>\gamma$;
by convention, output is NO if $\upr{\mathbf{E}}=0$. 
\end{description}

So, we are ready to state our main results on complexity for the credal semantics.
To understand these results, consider the computation of lower probabilities by
the algorithms in Section \ref{section:Structure}: the basic idea is to go through
all possible configurations of probabilistic facts, and each configuration requires
  runs of cautious/brave inference (that it, it is necessary to check whether all possible
stable models satisfy $\mathbf{Q} \cap \mathbf{E}$, and whether all possible
stable models fail to satisfy $\mathbf{Q}$ while satisfying $\mathbf{E}$. 
Thus the proof strategies employed previously can be adapted to some extent,
by using cautious/brave inference in our Turing machines. We have:

\begin{Theorem}
\label{theorem:CredalComplexity} 
Adopt the credal semantics for {\sc plp}s, and assume that input {\sc plp}s are consistent.
The lower-probability version of inferential complexity is 
$\mathsf{PP}^{\mathsf{NP}^\mathsf{NP}}$-complete for \textsc{plp}s where all predicates have a bound on arity, 
and $\mathsf{PP}^\mathsf{NP}$-complete for propositional \textsc{plp}s.
The lower-probability version of query complexity is $\mathsf{PP}^\mathsf{NP}$-complete.
\end{Theorem}
\begin{proof}
We first focus on propositional programs.

To prove membership, we describe a polynomial time nondeterministic Turing machine such that 
more than half of its computation paths, on a given input, end up accepting iff the input is an YES instance. 
The machine receives the {\sc plp} $\left<\mathbf{P},\mathbf{PF}\right>$, the pair $(\mathbf{Q},\mathbf{E})$,
and the rational $\gamma\in[0,1]$ as input.
In case $\lpr{\mathbf{Q},\mathbf{E}}+\upr{\neg \mathbf{Q},\mathbf{E}}>0$,
where we use $\neg \mathbf{Q}$ to indicate that $\mathbf{Q}$ does not hold, we have to decide whether:
\[
\frac{\lpr{\mathbf{Q},\mathbf{E}}}{\lpr{\mathbf{Q},\mathbf{E}}+\upr{\neg \mathbf{Q},\mathbf{E}}} > \gamma
\quad \Leftrightarrow \quad
(1-\gamma) \lpr{\mathbf{Q},\mathbf{E}} > \gamma \upr{\neg \mathbf{Q},\mathbf{E}}.
\]
Write $\gamma$ as $\mu/\nu$ for the smallest possible integers $\mu$ and $\nu$, such that $\nu>0$, 
to conclude that our decision is whether 
\begin{equation}
\label{equation:BasicCheck}
(\nu-\mu) \lpr{\mathbf{Q},\mathbf{E}} > \mu \upr{\neg \mathbf{Q},\mathbf{E}}.
\end{equation}

In case $\lpr{\mathbf{Q},\mathbf{E}}+\upr{\neg \mathbf{Q},\mathbf{E}}=0$,
there are a few cases to consider, as indicated by the discussion around 
Expression (\ref{equation:BayesRuleA}).
First, if $\upr{\mathbf{Q}\cap\mathbf{E}}=0$, the machine must return NO (numeric
probability value is not defined);
and if $\upr{\mathbf{Q}\cap\mathbf{E}}>0$, the machine must return NO if $\gamma=1$
and YES if $\gamma<1$.  
One simple way to capture all these cases is this: 
if $\upr{\mathbf{Q},\mathbf{E}}>0$ and $\upr{\neg \mathbf{Q},\mathbf{E}}=0$ and $\gamma<1$,
then return YES and stop; otherwise return YES or NO according to 
inequality in Expression (\ref{equation:BasicCheck}). 
Thus the machine starts by handling the special case in the previous sentence.
If $\gamma<1$, then the machine determines whether 
$\upr{\mathbf{Q} \cap \mathbf{E}}>0$ and 
$\upr{\neg \mathbf{Q} \cap \mathbf{E}}=0$ using the $\mathsf{NP}$ oracle twice.
In each case, the oracle guesses a total choice and determines, using brave inference,
whether there is a stable model that satisfies the event of interest. If there is no
such total choice, then the upper probability is zero. 
So, if $\gamma<1$ and 
$\upr{\mathbf{Q} \cap \mathbf{E}}>0$ and 
$\upr{\neg \mathbf{Q} \cap \mathbf{E}}=0$, 
move into the accepting state; otherwise, move to some state $q$ and continue. 

From $q$, the machine ``goes through'' the possible selections of probabilistic facts,
operating similarly to the second algorithm in Section \ref{section:Structure}. We will use the fact 
that cautious logical reasoning is $\mathsf{coNP}$-complete and
brave logical reasoning is $\mathsf{NP}$-complete  \cite[Table 2]{Eiter2007}.

The machine proceeds from $q$ as in the proof of 
Theorem \ref{theorem:AcyclicBounded}, nondeterministically
selecting whether each fact is kept or discarded. Suppose we have $n$ ground probabilistic facts
$\alpha_1::A_1\bigperiod,\dots,\alpha_n::A_n\bigperiod$. For each probabilistic fact $\alpha_i::A_i\bigperiod$, 
where $\alpha_i=\mu_i/\nu_i$ for smallest integers $\mu_i$ and $\nu_i$ such that
$\nu_i>0$, the machine creates $\mu_i$ computation paths 
out of the decision to keep $A_i$, and $\nu_i-\mu_i$ 
computation paths out of the decision to discard $A_i$. 
Note that after guessing the status of each probabilistic fact the machine may branch in at most $\nu_i$
paths, and the total number of paths out of this sequence of decisions is $\prod_{i=1}^n \nu_i$. 
Denote this latter number by $N$. At this point the machine has a normal logic program, and it
runs cautious inference to determine whether $\mathbf{Q}\cap\mathbf{E}$ holds in every
stable model of this program. Cautious logical reasoning is solved by the $\mathsf{NP}$ oracle.  If indeed 
$\mathbf{Q}\cap\mathbf{E}$ holds in every stable model of this program, the machine moves to state $q_1$.
Otherwise, the machine runs brave inference to determine whether $\mathbf{Q}$ is $\mathsf{false}$ while
$\mathbf{E}$ is $\mathsf{true}$ in some stable model of the program. Brave logical reasoning
is solved by the $\mathsf{NP}$ oracle.
And if indeed $\mathbf{Q}$ is $\mathsf{false}$ while
$\mathbf{E}$ is $\mathsf{true}$ in some stable model of the program, the machine moves to state $q_2$.
Otherwise, the machine moves to state $q_3$. Denote by $N_1$ the number of computation paths that arrive
at $q_1$, and similarly for $N_2$ and $N_3$. From $q_1$ the machine branches into $\nu-\mu$ 
computation paths that all arrive at the accepting state (thus there are $(\nu-\mu)N_1$ 
paths through $q_1$ to the accepting
state). And from $q_2$ the machine branches into $\mu$ computation 
paths that all arrive at the rejecting state.
Finally, from $q_3$ the machine nondeterministically moves either into the accepting or the rejecting state.
Thus the  number of accepting computation paths is larger than the number of rejecting computation paths iff
\[
(\nu-\mu)N_1 + N_3 > \mu N_2 + N_3 \quad \Leftrightarrow \quad 
(\nu-\mu) \frac{N_1}{N} > \mu \frac{N_2}{N}.
\]
Note that, by construction, $N_1/N = \lpr{\mathbf{Q},\mathbf{E}}$ and 
$N_2/N = \upr{\neg \mathbf{Q},\mathbf{E}}$; thus 
the  number of accepting computation paths is larger than the number of rejecting computation paths iff
\[
(\nu-\mu)  \lpr{\mathbf{Q},\mathbf{E}}  > \mu \upr{\neg \mathbf{Q},\mathbf{E}}.
\]
Membership is thus proved.

Hardness is shown by a reduction from the problem $\#_1\mathsf{DNF}(>)$: Decide whether 
the number  of assignments of $\mathbf{X}$ such that the formula
$\phi(\mathbf{X}) = \forall  \mathbf{Y} : \varphi(\mathbf{X},\mathbf{Y})$ holds is strictly larger than $M$,
where $\varphi$ is a propositional formula in DNF with conjuncts $d_1,\ldots,d_k$
(and $\mathbf{X} = \{x_1,\dots,x_n\}$ and $\mathbf{Y}=\{y_1,\dots,y_m\}$ 
are sets of propositional variables).
Introduce $\mathsf{x_i}$ for each $x_i$ and $\mathsf{y_i}$ for each $y_i$, and
encode $\phi$ as follows. Each conjunct $d_j$ is represented by a predicate $\mathsf{d_j}$
and a rule $\mathsf{d_j} \colonminus s_1, \dots, s_r\bigperiod$, where $s_i$ stands for
a properly encoded subgoal: either some $\mathsf{x_i}$, or $\mathbf{not}\;\mathsf{x_i}$,
or some $\mathsf{y_i}$, or $\mathbf{not}\;\mathsf{y_i}$. And then introduce $k$ rules
$\mathsf{dnf} \colonminus \mathsf{d_j}\bigperiod$, one per conjunct. 
Note that for a fixed truth assignment for all $\mathsf{x_i}$ and all $\mathsf{y_i}$,
$\mathsf{dnf}$ is $\mathsf{true}$ iff $\varphi$ holds. Now introduce probabilistic
facts $0.5::\mathsf{x_i}$, one for each $\mathsf{x_i}$. There are then $2^n$ possible
ways to select probabilistic facts. The remaining problem is to encode the univeral
quantifier over the $y_i$. To do so, introduce a pair of rules for each $\mathsf{y_i}$,
\[
\mathsf{y_i} \colonminus \mathbf{not}\;\mathsf{ny_i}\bigperiod \quad \mbox{ and } \quad
\mathsf{ny_i} \colonminus \mathbf{not}\;\mathsf{y_i}\bigperiod.
\]
Thus there are $2^m$ stable models running through assignments of $y_1,\dots,y_m$,
for each fixed selection of probabilistic facts.
By Expression (\ref{equation:Bounds}) we have that $\lpr{\mathsf{dnf}=\mathsf{true}}$
is equal to $\sum_\theta \min f(\theta)/2^n$, where $\theta$ denotes a total choice,
the minimum is over all stable models produced by $\mathbf{P}\cup\mathbf{PF}^{\downarrow\theta}$,
and $f(\theta)$ is a function that yields $1$ if $\mathsf{dnf}$ is $\mathsf{true}$
and $0$ otherwise. Now $\min f(\theta)$ yields $1$ iff
for all $\mathbf{Y}$ we have that $\varphi(\mathbf{X},\mathbf{Y})$ is $\mathsf{true}$,
where $\mathbf{X}$ is fixed by $\theta$. 
Hence $\lpr{\mathsf{dnf}=\mathsf{true}}>M/2^n$ iff the input problem is accepted. 
Hardness is thus proved.
%%%%%%%%%%%%%%%%

Now consider {\sc plp}s where predicates have bounded arity.

Membership follows using the same construction described for the propositional case,
but using a $\Sigma_2^P$ oracle as cautious logical reasoning is $\Pi_2^P$-complete 
and brave logical reasoning is $\Sigma_2^P$-complete \cite[Table 5]{Eiter2007}.

Hardness is shown by a reduction from $\#_2\mathsf{3CNF}(>)$:
Decide whether the number of assignments of $\mathbf{X}$ such that the formula
$\phi(\mathbf{X}) = \forall \mathbf{Z} : \exists \mathbf{Y} : \varphi(\mathbf{X},\mathbf{Y},\mathbf{Z})$
holds is strictly larger than $M$, where $\varphi$ is a propositional formula 
in 3CNF with clauses $c_1,\dots,c_k$ (and $\mathbf{X}$, $\mathbf{Y}$,
and $\mathbf{Z}$ are sets of propositional variables, and $\mathbf{X}$ contains 
$n$ propositional variables). We proceed {\em exactly} as in the proof of hardness 
for Theorem \ref{theorem:AcyclicBounded}; each propositional variable $y_i$ now appears as a logical
variable $Y_i$, while each propositional variable $x_i$ appears as a predicate $\mathsf{x_i}$.
The novelty is that each propositional variable $z_i$ appears as a predicate $\mathsf{z_i}$ that
receive the same treatment as predicates $\mathsf{y_i}$ in the proof for the propositional case.
So, just repeat the whole translation of the formula $\varphi$ used in the proof 
of Theorem \ref{theorem:AcyclicBounded}, with the only difference that now there may be 
propositional variables $z_i$ in the formula, and these propositional variables appear as 
predicates $\mathsf{z_i}$ in the {\sc plp}. Then introduce, for each $z_i$, a pair of rules 
\[
\mathsf{z_i} \colonminus \mathbf{not}\;\mathsf{nz_i}\bigperiod \quad \mbox{ and } \quad
\mathsf{nz_i} \colonminus \mathbf{not}\;\mathsf{z_i}\bigperiod.
\]
Again, for each fixed selection of probabilistic facts, there are stable models, 
one per assignment of $\mathbf{Z}$. And $\lpr{\mathsf{cnf}=\mathsf{true}}>M/2^n$ 
iff the input problem is accepted. 

%%%%%%%%
Finally, consider query complexity.

Membership follows using the same construction described for the propositional case,
but using a $\mathsf{NP}$ oracle as cautious logical reasoning is $\mathsf{coNP}$-complete 
and brave logical reasoning is $\mathsf{NP}$-complete  \cite[Theorem 5.8]{Dantsin2001}. 

Hardness follows again by a reduction from $\#_1\mathsf{DNF}(>)$; that is, again
we must decide whether the number of assignments of $\mathbf{X}$ such that 
$\forall \mathbf{Y} :  \varphi(\mathbf{X},\mathbf{Y})$ holds is strictly larger than $M$,
where $\varphi$ is a formula in DNF (again, the number of propositional variables in $\mathbf{X}$
is $n$, and the number of propositional variables in $\mathbf{Y}$ is $m$). 
We employ a construction inspired by the proof of Theorem \ref{theorem:SingleRule},
using the following fixed {\sc plp}. Note that $\mathsf{x}$ stands for the propositional variables
in $\mathbf{X}$, where counting operates; $\mathsf{y}$ stands for the propositional
variables in $\mathbf{Y}$, where the universal quantifier operates; $\mathsf{c}$
stands for clauses that are then negated to obtain the DNF:
\[
\begin{array}{c}
0.5::\mathsf{x}(V)\,\bigperiod \\
0.5::\mathsf{select1}(U,V)\bigperiod \qquad
0.5::\mathsf{select2}(U,V)\bigperiod \\
0.5::\mathsf{select3}(U,V)\bigperiod \qquad
0.5::\mathsf{select4}(U,V)\bigperiod \\
\mathsf{y}(V) \colonminus \mathbf{not}\;\mathsf{ny}(V)\bigperiod \qquad 
 \mathsf{ny}(V) \colonminus \mathbf{not}\;\mathsf{y}(V)\bigperiod \\ 
\mathsf{c}(V) \colonminus \mathsf{select1}(U,V), \mathsf{x}(U)\,\bigperiod \\
\mathsf{c}(V) \colonminus \mathsf{select2}(U,V), \mathbf{not}\;\mathsf{x}(U)\,\bigperiod \\
\mathsf{c}(V) \colonminus \mathsf{select3}(U,V), \mathsf{y}(U)\,\bigperiod \\
\mathsf{c}(V) \colonminus \mathsf{select4}(U,V), \mathbf{not}\;\mathsf{y}(U)\,\bigperiod \\
\mathsf{d}(V) \colonminus \mathbf{not}\;\mathsf{c}(V)\bigperiod \\
\mathsf{aux} \colonminus \mathsf{d}(V)\bigperiod \\
\mathsf{dnf} \colonminus \mathbf{not}\;\mathsf{aux}\bigperiod 
\end{array}
\]
By providing $\mathsf{select1}$, $\mathsf{select2}$, $\mathsf{select3}$ and $\mathsf{select4}$
as $\mathbf{Q}$, we can encode the formula $\varphi$. Add to $\mathbf{Q}$ the assignment
$\{\mathsf{dnf}=1\}$, and then $\lpr{\mathbf{Q}} > M/2^{4s^2+s}$, where $s=\max(m,n)$,
iff the input problem is accepted. 
$\Box$
\end{proof}

Theorem \ref{theorem:CredalComplexity}  focuses on the computation of lower probabilities.
We can of course define the {\em upper-probability} versions of inferential
and query complexities, by replacing the decision $\lpr{\mathbf{Q}|\mathbf{E}}>\gamma$
with $\upr{\mathbf{Q}|\mathbf{E}}>\gamma$. 
If anything, this latter decision leads to easier proofs of membership, for all special
cases are dealt with by deciding whether
\[
(\nu-\mu) \upr{\mathbf{Q},\mathbf{E}} > \mu \lpr{\neg\mathbf{Q},\mathbf{E}},
\]
where again $\gamma=\mu/\nu$. All other points in the membership proofs remain
the same, once brave and cautious reasoning are exchanged. 
Several arguments concerning hardness can also be easily adapted. 
For instance, $\mathsf{PP}^\mathsf{NP}$-hardness for propositional
programs can be proved by reducting from $\#_1\mathsf{CNF}(>)$, by encoding a formula in CNF.
Similarly, $\mathsf{PP}^\mathsf{NP}$-hardness for query complexity reduces from 
$\#_1\mathsf{CNF}(>)$ by using the fixed program described in the proof of 
Theorem \ref{theorem:CredalComplexity} without the latter three rules  
(and query with assignments on groundings of~$\mathsf{c}$).  
%%% FABIO: OPEN PROBLEM IS TO PROVE PP^NP^NP FOR 
%%%            UPPER PROBABILITY, RELATIONAL WITH BOUNDED ARITY.
%%%            THE MAIN QUESTION IS WHETHER #_2 DNF CAN BE RESTRICTED
%%%            TO 3DNF...

Theorem \ref{theorem:CredalComplexity}  does not discuss the complexity of {\sc plp}s,
under the credal semantics,  {\em without} a bound on arity. Without such a bound, 
{\em logical} cautious reasoning is $\mathsf{coNEXP}$-complete, so we conjecture that 
exponentially bounded counting Turing machines will be needed here. We leave this
conjecture as an open question.

Finally, our complexity results were obtained assuming that \textsc{plp}s were consistent;
of course, in practice one must  consider the problem of checking consistency.
% (that is, the problem of determining whether a given {\sc plp} has semantics).  
We have:

\begin{Proposition}
Consistency checking is $\Pi_2^P$-complete for propositional {\sc plp}s
and is $\Pi_3^P$-complete for {\sc plp}s where predicates have a bound on arity.
\end{Proposition}
\begin{proof}
Membership of consistency checking of a propositional {\sc plp} obtains by verifying
whether logical consistency holds for each total choice of
probabilistic facts, and this can be accomplished by deciding whether
all total choices satisfy logical consistency (logical consistency
  checking for this language is $\mathsf{NP}$-complete \cite[Table 1]{Eiter2007}).
An analogue reasoning  leads to membership in $\Pi_3^P$ for {\sc plp}s
with a bound on arity, as logical consistency checking with bounded arity is 
$\Sigma^P_2$-complete \cite[Table 4]{Eiter2007}.

Now consider hardness in the propositional case. Take a sentence $\phi$ equal to
$\forall \mathbf{X} : \exists \mathbf{Z} : \varphi(\mathbf{X},\mathbf{Z})$, 
where $\phi$ is a propositional formula in 3CNF with clauses $c_1,\dots,c_k$, and
vectors of propositional variables $\mathbf{X}$ and $\mathbf{Z}$.
Deciding the satisfiability of such a formula is a $\Pi^P_2$-complete problem \cite{Marx2011}.
So, introduce a predicate $\mathsf{x_i}$ for each $x_i$, associated with a
probabilistic fact $0.5::\mathsf{x_i}\bigperiod$, and a predicate
$\mathsf{z_i}$ for each $z_i$, associated with rules
\[
\mathsf{z_i} \colonminus \mathbf{not}\;\mathsf{nz_i}\bigperiod \quad \mbox{ and } \quad
\mathsf{nz_i} \colonminus \mathbf{not}\;\mathsf{z_i}\bigperiod.
\]
Now encode the formula $\phi$ as follows. For each clause $c_j$ with three literals,
add the rules 
$\mathsf{c_j} \colonminus \ell_1\bigperiod$, 
$\mathsf{c_j} \colonminus \ell_2\bigperiod$, and 
$\mathsf{c_j} \colonminus \ell_3\bigperiod$,
where each $\ell_i$ stands for a subgoal containing a predicate in 
$\mathsf{x_1},\dots,\mathsf{x_n}$ or in $\mathsf{z_1},\dots,\mathsf{z_m}$, 
perhaps preceded by $\mathbf{not}$, as appropriate (mimicking a similar
construction in the proof of Theorem \ref{theorem:AcyclicBounded}).
Then add a rule
\[
\mathsf{cnf} \colonminus \mathsf{c_1},\dots,\mathsf{c_k}\bigperiod
\]
to build the formula $\varphi$, and an additional rule
\[
\mathsf{clash} \colonminus \mathbf{not}\;\mathsf{clash}, \mathbf{not}\;\mathsf{cnf}\bigperiod
\]
to force $\mathsf{cnf}$ to be $\mathsf{true}$ in any stable model. 
The question of whether this program has a stable model for every configuration
of $\mathbf{X}$ then solves the original question about satisfiability of $\phi$.

Finally, consider hardness in the relational (bounded arity) case. Take a sentence $\phi$ equal to 
$\forall \mathbf{X} : \exists \mathbf{Z} : \neg \exists \mathbf{Y} : \varphi(\mathbf{X},\mathbf{Y},\mathbf{Z})$,
where $\varphi$ is a propositional formula in 3CNF; deciding the satisfiability of this formula
is a $\Pi^P_3$-complete problem \cite{Marx2011}. 
Denote $\neg \exists \mathbf{Y} : \phi(\mathbf{X},\mathbf{Y},\mathbf{Z})$ by $\phi'$. 
The strategy here will be to combine the constructs in the previous paragraph (propositional case)
with the proof of hardness for Theorem \ref{theorem:AcyclicBounded}. That is, 
introduce a predicate $\mathsf{x_i}$ for each $x_i$, associated with a
probabilistic fact $0.5::\mathsf{x_i}\bigperiod$, and a predicate
$\mathsf{z_i}$ for each $z_i$, again associated with rules
\[
\mathsf{z_i} \colonminus \mathbf{not}\;\mathsf{nz_i}\bigperiod \quad \mbox{ and } \quad
\mathsf{nz_i} \colonminus \mathbf{not}\;\mathsf{z_i}\bigperiod.
\]
And then encode each clause $c_j$ of $\varphi$ by introducing a predicate $\mathsf{c_j}(\mathbf{Y}_j)$,
where $\mathbf{Y}_j$ is exactly as in the proof of Theorem \ref{theorem:AcyclicBounded}. And
as in that proof, introduce
\[
\mathsf{cnf} \colonminus \mathsf{c_1}(\mathbf{Y}_1), \mathsf{c_2}(\mathbf{Y}_2), 
     \dots, \mathsf{c_m}(\mathbf{Y}_m)\bigperiod
\]
and force $\phi'$ to be $\mathsf{false}$ by introducing:
\[
\mathsf{clash} \colonminus \mathbf{not}\;\mathsf{clash}, \mathsf{cnf}\bigperiod.
\]
The question of whether this program has a stable model for every configuration of
$\mathbf{X}$ then solves the original question about satisfiability of $\varphi$.
\end{proof}

Again we have left open the complexity of consistency checking for {\sc plp}s without
a bound on predicate arity. This question should be addressed in future work.

\section{The complexity of inference under the well-founded semantics}\label{section:ComplexityWellFounded}

In this section we investigate the complexity of {\em probabilistic} inference under the
well-founded semantics.  As before, we examine propositional and relational programs,
and within the latter we look at programs with a bound on predicate arity.
Note that a bound on predicate arity forces each predicate to have a polynomial 
number of groundings, but the grounding of the program may still be exponential 
(as there is no bound on the number of atoms that appear in a single rule, each 
rule may have many logical variables, thus leading to many groundings). 
 
\begin{Theorem}\label{theorem:Complexity}
Adopt the well-founded semantics for {\sc plp}s. 
The inferential complexity of {\sc plp}s is $\mathsf{PEXP}$-complete;
it is $\mathsf{PP}^\mathsf{NP}$-complete if the {\sc plp} has a bound on the arity of its predicates;
it is $\mathsf{PP}$-complete if the {\sc plp} is propositional. 
The query complexity of {\sc plp}s is $\mathsf{PP}$-complete.
\end{Theorem}
\begin{proof}
Consider first propositional {\sc plp}s. Such a {\sc plp} can encode any Bayesian network 
over binary variables \cite{Poole93AI}, so inference is $\mathsf{PP}$-hard. Membership 
is proved by adapting the arguments in the proof of Theorem \ref{theorem:AcyclicBounded};
whenever a total choice is selected by the nondeterministic Turing machine, logical inference
(under the well-founded semantics) is run with polynomial effort in the resulting propositional 
normal logic program  \cite{Dantsin2001}.

Consider now {\sc plp}s with logical variables. 
Membership follows from the same argument in the previous paragraph, using the 
fact that inference in normal logic programs under the well-founded semantics is 
in $\mathsf{EXP}$ \cite{Dantsin2001}. Hardness follows from the fact that inferential 
complexity is $\mathsf{PEXP}$-hard already for acyclic programs
(Theorem \ref{theorem:AcyclicEXP}). 

Now consider {\sc plp}s with a bound on the arity of predicates. 
Membership follows from the same argument in the previous paragraphs, using 
the fact that inference in normal logic programs with a bound on the arity of 
predicates is  under the well-founded semantics is in in $\mathsf{P}^\mathsf{NP}$, 
as proved in Theorem \ref{theorem:Logical}.
Hardness for $\mathsf{PP}^\mathsf{NP}$ follows from the fact that inference complexity 
of {\sc plp}s under the stable model semantics is $\mathsf{PP}^\mathsf{NP}$-hard even 
for stratified programs, and noting that for stratified programs the
stable model and the well-founded semantics agree.
\end{proof}

The proof of Theorem \ref{theorem:Complexity}, in the case of {\sc plp}s with bound
on predicate arity, uses the following result. Note that this is a result on logical 
inference; however it does not seem to be found in current literature.
 
\begin{Theorem}\label{theorem:Logical}
Consider the class of normal logic programs with a bound on the arity of predicates,
and consider the problem of deciding whether a literal is in the well-founded model
of the program. This decision problem is $\mathsf{P}^\mathsf{NP}$-complete.
\end{Theorem}
\begin{proof}
Hardness follows from the hardness of logical inference with stratified programs
under the stable model semantics \cite{Eiter2007}. 
Membership requires more work. 
We use the monotone operator 
$\mathbb{LFT}_\mathbf{P}(\mathbb{LFT}_\mathbf{P}(\mathcal{I}))$. Consider
the algorithm that constructs the well-founded extension by starting with the empty
interpretation and by iterating $\mathbb{LFT}_\mathbf{P}(\mathbb{LFT}_\mathbf{P}(\mathcal{I}))$. 
As there are
only polynomially-many groundings, there are at most a polynomial number of
iterations.  Thus in essence we need to iterate the operator
$\mathbb{LFT}_\mathbf{P}(\mathcal{I})$; thus, focus attention on the computation
of $\mathbb{LFT}_\mathbf{P}(\mathcal{I})$. The latter computation 
consists of finding the least fixpoint of $\mathbb{T}_{\mathbf{P}^\mathcal{I}}$.
So we must focus on the effort involved in computing
the least fixpoint of $\mathbb{T}_{\mathbf{P}^\mathcal{I}}$. 
Again, there are at most a polynomial number of iterations of $\mathbb{T}_{\mathbf{P}^\mathcal{I}}$
to be run. So, focus on a single iteration of $\mathbb{T}_{\mathbf{P}^\mathcal{I}}$.
Note that any interpretation $\mathcal{I}$ has polynomial size; however, we cannot
explicitly generate the reduct $\mathcal{P}^\mathcal{I}$ as it may have exponential size.
What we need to do then is, for each grounded
atom $A$, to decide whether there is a rule whose grounding makes the atom $A$
$\mathsf{true}$ in $\mathbb{T}_{\mathbf{P}^\mathcal{I}}$. So we must make a 
nondeterministic choice per atom (the choice has the size of logical variables in a 
rule, a polynomial number). Hence  by running a polynomial number of nondeterministic
choices, we obtain an iteration of $\mathbb{T}_{\mathbf{P}^\mathcal{I}}$; by running
a polynomial number of such iterations, we obtain a single iteration
of $\mathbb{LFT}_\mathbf{P}(\mathcal{I})$; 
and by running a polynomial number of
such iterations, we build the well-founded model. Thus we are within
$\mathsf{P}^\mathsf{NP}$ as desired.
\end{proof}

Obviously, for the well-founded semantics there are no concerns about consistency:
{\em every} normal logic program has a well-founded semantics, so {\em every} {\sc plp}
has one and only one well-founded semantics.

\section{Conclusion}
\label{section:Conclusion}

We can summarize our contributions as follows. First, we have identified and compared
the main ideas between the credal and the well-founded semantics for {\sc plp}s based
on probabilistic facts and normal logic programs. Other semantics may be studied in 
future work, but the credal and the well-founded ones seem to be the most important
starting point. Second, we have shown that the credal semantis is intimately related
to infinitely monotone Choquet capacitites; precisely: the credal semantics of a consistent
{\sc plp}  is a credal set that dominates a infinitely monotone Choquet capacity. 
Third, we have derived the inferential and query complexity of acyclic, stratified and
general {\sc plp}s both under the credal and the well-founded semantics. 
These results on complexity are summarized in Table \ref{table:Complexity}; note that
{\sc plp}s reach non-trivial classes in the counting hierarchy.
It is interesting to note that acyclic {\sc plp}s with a bound on arity go beyond Bayesian
networks in the complexity classes they can reach. 

\begin{table}
\begin{center}
\begin{tabular}{|c||c|c|c|c|} \hline 
            & Propositional & $\mathsf{PLP_b}$ & $\mathsf{PLP}$ & Query \\
\hline \hline & & & & \\[-4mm]
Acyclic & $\mathsf{PP}$ & $\mathsf{PP}^\mathsf{NP}$ & $\mathsf{PEXP}$ & $\mathsf{PP}$ \\
\hline & & & & \\[-4mm]
Stratified & $\mathsf{PP}$ & $\mathsf{PP}^\mathsf{NP}$ & $\mathsf{PEXP}$ & $\mathsf{PP}$ \\
\hline & & & & \\[-4mm]
General, credal & $\mathsf{PP}^\mathsf{NP}$ & $\mathsf{PP}^{\mathsf{NP}^\mathsf{NP}}$ & ? & $\mathsf{PP}^\mathsf{NP}$ \\
\hline & & & & \\[-4mm]
General, well-founded & $\mathsf{PP}$ & $\mathsf{PP}^\mathsf{NP}$ & $\mathsf{PEXP}$ & $\mathsf{PP}$ \\
\hline
\end{tabular}
\end{center}
\caption{Complexity results. All entries refer to completeness with respect to many-one reductions.
Columns ``Propositional'',  ``$\mathsf{PLP_b}$'', and ``$\mathsf{PLP}$
respectively refer to the inferential complexity of propositional {\sc plp}s, the inferential
complexity of {\sc plp}s with a bound on predicate arity, and {\sc plp}s with no bound on
predicate arity. Column ``Query'' refers to the query complexity of relational {\sc plp}s.}
\label{table:Complexity}
\end{table}

Concerning complexity, acyclic and stratified {\sc plp}s have identical credal and well-founded
semantics, while general {\sc plp}s may have different credal and well-founded semantics. 
For normal logic programs (not probabilistic ones), the well-founded semantics is known
to stay within lower complexity classes than the credal semantics \cite{Dantsin2001};
the same phenomenon persists in the probabilistic case. Indeed, the well-founded semantics
for general {\sc plp}s reaches the same complexity classes as for acyclic {\sc plp}s. One
might take this as an argument for the well-founded semantics, on top of the fact that
the well-founded semantics is defined for {\em any} {\sc plp}. On the other hand, our
analysis in Section \ref{section:Semantics} favors, at least conceptually, the credal semantics,
despite the fact that it may not be defined for some {\sc plp}s (in fact one might argue
that no semantics should be defined for such {\sc plp}s). It is much easier to understand the
meaning of {\sc plp}s using the credal semantics than the well-founded semantics, as the
latter mixes three-valued logic and probabilities in a non-trivial way. 
We suggest that more study is needed to isolate those programs where $\mathsf{undefined}$ values
are justified and can be properly mixed with probabilities. Also, the well-founded semantics 
may be taken as an approximation of the set of possible probability models. 
In any case, we find that Lukasiewicz's credal semantics is quite attractive and not as well known
as it deserves to be.

We could include in the analysis of {\sc plp}s a number of useful
constructs that have been adopted in answer set programming  \cite{Eiter2009Primer}.
There, {\em classic negation}, such as $\neg \mathsf{wins}(X)$,
%{\em constraints} such as $\colonminus \phi$, and {\em disjunctive heads} are typically allowed.
is allowed on top of $\mathbf{not}$. Also, constraints, such as $\colonminus \phi$, 
are allowed to mean that $\phi$ is $\mathsf{false}$.
%But classic negation and constraints are syntactic sugar when cyclic programs 
%are allowed, so we can adopt them if so desired. 
More substantial is the presence,
in answer set programming, of {\em disjunctive heads}. 
With such a machinery,
we can for instance rewrite the rules in Example~\ref{example:ProbabilisticBasic} as a single rule
$\mathsf{single}(X) \vee \mathsf{husband}(X) \colonminus \mathsf{man}(X).$,
and the rules in Example~\ref{example:GraphColoring} as the pair:
\[
\begin{array}{c}
\mathsf{color}(V,\mathsf{red}) \vee
\mathsf{color}(V,\mathsf{yellow}) \vee
\mathsf{color}(V,\mathsf{green}) \colonminus \mathsf{vertex}(V). \\
\colonminus  
	\mathsf{edge}(V,U), \mathsf{color}(V,C), \mathsf{color}(U,C). 
\end{array}
\]

Now the point to be made is this. Suppose we have a probabilistic logic program
$\left<\mathbf{P},\mathbf{PF}\right>$, where as before we have independent
probabilistic facts, but where $\mathbf{P}$ is now a logic program with classic
negation, constraints, disjuctive heads, and $\mathbf{P}$ is consistent
in that it has stable models for every total choice of probabilistic facts.
The proof of
Theorem \ref{theorem:Capacity} can be reproduced in this setting, and hence 
{\em the credal semantics (the set of measures over stable models)
of these probabilistic answer set programs is again an infinite monotone credal set.}
The complexity of inference with these constructs is left for future investigation.

Much more is yet to be explored concerning the complexity of {\sc plp}s. 
Several classes of {\sc plp}s deserve attention, such as definite, tight, strict, 
order-consistent programs, and programs with aggregates and other constructs.
% {\em tight} programs (for those, the Clark completion determines the stable models.
%% AGGREGATES: NEEDED TO ENCODE PRMs as indicated by Fierens
The inclusion of functions (with appropriate restrictions to ensure decidability) is 
another challenge.  
Concerning complexity theory itself, it seems that approximability should be investigated,
as well as questions surrounding learnability and expressivity of {\sc plp}s.

\section*{Acknowledgements}

The first author is partially supported by CNPq, grant 308433/2014-9. 
The second author received financial support from the S\~ao Paulo Research Foundation (FAPESP),
 grant 2016/01055-1. 
 
\bibliographystyle{theapa}
%\bibliography{/Users/imac/Desktop/HOME/Research/Bibliography/consulted}

\end{document}